\documentclass{new_tlp}
\usepackage{mathptmx}

\usepackage{soul}
\usepackage{url}
\usepackage[utf8]{inputenc}
\usepackage{amsmath}
\usepackage{amssymb}
\usepackage{booktabs}
\usepackage{algorithm}
\usepackage[noend]{algpseudocode}
\algrenewcommand\algorithmicrequire{\textbf{Input:}}
\algrenewcommand\algorithmicensure{\textbf{Output:}}
\usepackage{listings}
\usepackage{tikz}
\usepackage{tikzscale}
\usetikzlibrary{arrows}
\usetikzlibrary{shapes.geometric}
\usetikzlibrary{shapes.geometric}
\usetikzlibrary{arrows,backgrounds}
\usetikzlibrary{fit}
\usetikzlibrary{positioning}
\usepackage{color}

\usepackage{pdflscape}

\newtheorem{example}{Example}

\newtheorem{proposition}{Proposition}
\newtheorem{definition}{Definition}

\usepackage{fancyvrb}
\newcommand\bcmdtab{\noindent\bgroup\tabcolsep=0pt%
  \begin{tabular}{@{}p{10pc}@{}p{20pc}@{}}}
\newcommand\ecmdtab{\end{tabular}\egroup}

\title[Generating Global and Local Explanations for Tree-Ensemble Learning Methods by ASP]
{Generating Global and Local Explanations for Tree-Ensemble Learning Methods by Answer Set Programming}

\author[A. Takemura and K. Inoue]
    {AKIHIRO TAKEMURA\\
    The Graduate University for Advanced Studies, SOKENDAI, Tokyo, Japan\\
    National Institute of Informatics, 2-1-2 Hitotsubashi, Chiyoda-ku, Tokyo, Japan 101-8430\\
    \email{atakemura@nii.ac.jp}
    \and KATSUMI INOUE\\
    National Institute of Informatics, 2-1-2 Hitotsubashi, Chiyoda-ku, Tokyo, Japan 101-8430\\
    The Graduate University for Advanced Studies, SOKENDAI, Tokyo, Japan\\
    \email{inoue@nii.ac.jp}
    }

\jdate{March 2003}
\pubyear{2003}
\pagerange{\pageref{firstpage}--\pageref{lastpage}}
\doi{S1471068401001193}

\begin{document}

\label{firstpage}

\maketitle
\begin{abstract}
We propose a method for generating rule sets as global and local explanations for tree-ensemble learning methods using Answer Set Programming (ASP).
To this end, we adopt a decompositional approach where the split structures of the base decision trees are exploited in the construction of rules, which in turn are assessed using pattern mining methods encoded in ASP to extract explanatory rules. 
For global explanations, candidate rules are chosen from the entire trained tree-ensemble models, whereas for local explanations, candidate rules are selected by only considering rules that are relevant to the particular predicted instance.
We show how user-defined constraints and preferences can be represented declaratively in ASP to allow for transparent and flexible rule set generation, and how rules can be used as explanations to help the user better understand the models. 
Experimental evaluation with real-world datasets and popular tree-ensemble algorithms demonstrates that our approach is applicable to a wide range of classification tasks.
Under consideration in Theory and Practice of Logic Programming (TPLP).
\end{abstract}


\section{Introduction}

\textit{Interpretability} in machine learning is the ability to explain or to present in understandable terms to a human \cite{doshi-velezRigorousScienceInterpretable2017,millerExplanationArtificialIntelligence2019a,molnarInterpretableMachineLearning2020}. Interpretability is particularly important when, for example, the goal of the user is to gain knowledge from some form of explanations about the data or process through machine learning models, or when making high-stakes decisions based on the outputs from the machine learning models where the user has to be able to trust the models. \textit{Explainability} is another term that is often used interchangeably with interpretability, but some emphasize the ability to produce \textit{post-hoc explanations} for the black-box models \cite{rudinStopExplainingBlack2019}. For convenience, we shall use the term \textit{explanation} when referring to post-hoc explanations in this paper.

In this work,\footnote{Some parts of this paper were presented as a Technical Communications paper \cite{takemura_treetap} at the 37th International Conference on Logic Programming (ICLP 2021). The present paper newly describes a method to produce explanations for \textit{each predicted instance} (local explanation), in addition to the updated ASP encoding for the global explanation method. The experimental section reports new evaluation results of the updated methods on various datasets, including several additional datasets.} we address the problem of explaining trained tree-ensemble models by extracting meaningful rules from them. 
This problem is of practical relevance in business and scientific domains, where the understanding of the behavior of high-performing machine learning models and extraction of knowledge in human-readable form can aid users in the decision-making process. We use {\it Answer Set Programming (ASP)} \cite{gelfondStableModelSemantics1988,lifschitzWhatAnswerSet2008} to generate rule sets from tree-ensembles.
ASP is a declarative programming paradigm for solving difficult search problems. 
An advantage of using ASP is its expressiveness and extensibility, especially when representing constraints. 
To our knowledge, ASP has never been used in the context of rule set generation from tree-ensembles, although it has been used in pattern mining \cite{jarvisaloItemsetMiningChallenge2011a,guyetUsingAnswerSet,DBLP:conf/ijcai/GebserGQ0S16,paramonovHybridASPbasedApproach2019}.

Generating explanations for machine learning models is a challenging task, since it is often necessary to account for multiple competing objectives. 
For instance, if accuracy is the most important metric, then it is in direct conflict with explainability because accuracy favors specialization while explainability favors generalization. Any explanation method should also strive to imitate the behavior of learned models as to minimize misrepresentation of models, which in turn may result in misinterpretation by the user. 
While there are many explanation methods available (some are covered in Section \ref{sec:relatedwork}), we propose to use ASP as a medium to represent the user requirements declaratively and to quickly search feasible solutions for faster prototyping. 
By implementing a rule selection method as a post-processing step to model training, we aim to offer an off-the-shelf objective explanation tool which can be applied to existing processes with minimum modification, as an alternative to subjective manual rule selection.

To demonstrate the adaptability of our approach, we present implementations for both {\it global} and {\it local} explanations of learned  tree-ensemble models using our method. In general, \textit{global explanation} refers to descriptions of how the overall system works (also referred to as \textit{model explanation}), and \textit{local explanation} refers to specific descriptions of why a certain decision was made ({\it outcome explanation}) \cite{guidottiSurveyMethodsExplaining2019}. The global explanations are more useful in situations where the explanations behind the opaque model is needed, for example, when designing systems for faster detection of certain events such as credit issues or illnesses. In contrast, the local explanations are suitable, for example, when explaining the outcome of such systems to its users, since they are more likely to be interested in particular decisions that led to the outcome.

% \todo{uncomment}
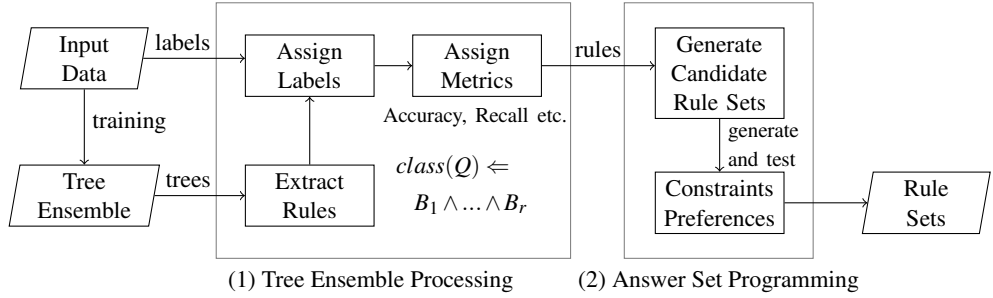
\begin{figure}[tb]
  \centering 
    \begin{tikzpicture}

\usetikzlibrary{shapes.geometric, arrows,backgrounds,fit,positioning}
\tikzset{
input/.style={trapezium, trapezium left angle=80, trapezium right angle=100, 
draw,  text centered, text width=1.2cm, minimum height=0.6cm},
input2/.style={trapezium, trapezium left angle=80, trapezium right angle=100,
draw,  text centered, text width=1.5cm, minimum height=0.6cm},
tp/.style={rectangle,  text centered, text width=7cm, minimum height=0.6cm},
tp1/.style={rectangle,  draw, text centered, text width=1.5cm, minimum height=0.6cm},
rule/.style={rectangle,  draw,  text centered, text width=7cm, minimum height=0.6cm},
asp/.style={rectangle,  draw,  text centered, text width=7cm, minimum height=0.6cm},
outputrule/.style={trapezium, trapezium left angle=60, trapezium right angle=120, draw,  text centered, text width=5cm, minimum height=0.6cm},
};

%% Change distance from EPTCS
\node[input] (a1) {\small Input Data};
\node[input2, below=1cm of a1] (a2) {\small Tree Ensemble};
\draw[->] (a1) -- (a2) node[midway,above,xshift=.6cm,yshift=-.2cm]{\small training};

\node[tp1, right=1.2cm of a2,] (b1) {\small Extract Rules};
\node[right=.05cm of b1, yshift=.3cm, text width=2.cm] (b1t) {\small \begin{align*}
% to make overfull warning go away
    cl&ass(Q) \Leftarrow\\ & B_1 \wedge ... \wedge  B_r
\end{align*}};
\node[tp1, right=1.5cm of a1.north, anchor=north, above=.9 of b1] (b2) {\small Assign Labels};
\node[tp1, right=.5cm of b2] (b3) {\small Assign Metrics};
\node[below=.03cm of b3, text width=2.5cm] (b3t) {\footnotesize Accuracy, Recall etc.};
\node[fit=(b1) (b2) (b3), draw=gray, text width=4.5cm,  minimum height=3.4cm] (teep) {};
\node[below=.05cm of teep, xshift=-0.3cm] (b) {\small (1) Tree Ensemble Processing};

\draw[->] (a1.east) -- (a1.east -| b2.west) node[midway,above,xshift=-.15cm]{\small labels};
\draw[->] (a2) -- (b1) node[midway,above,xshift=-.15cm]{\small trees};
\draw[->] (b1) -- (b2);
\draw[->] (b2) -- (b3);

\node[tp1, right=1.5cm of b3, yshift=-.1cm] (c1) {\small Generate Candidate Rule Sets};
\node[tp1, below=.7cm of c1] (c2) {\small Constraints Preferences};
\node[fit=(c1) (c2), draw=gray, text width=2.3cm, minimum height=3.4cm, right=.7cm of teep] (aspe) {};
\node[below=.05cm of aspe] (c) {\small (2) Answer Set Programming};

\draw[->] (b3.east) -- (b3.east -| c1.west) node[midway,above,]{\small rules};
\draw[->] (c1.south) -- (c2.north) node[midway,right,text width=1.3cm]{\footnotesize generate and test};

\node[input, right=1.1cm of c2] (d1) {\small Rule Sets};
\draw[->] (c2) -- (d1);

\end{tikzpicture}
    % \vspace{-.5\baselineskip}
  \caption{Overview of our framework}
  \label{fig:pipeline}
%   \vspace{-.5\baselineskip}
\end{figure}

We consider the two-step procedure for rule set generation from trained tree-ensemble models (Figure \ref{fig:pipeline}): (1) extracting rules from tree-ensembles, and (2) computing sets of rules according to selection criteria and preferences encoded declaratively in ASP. 
For the first step, we employ the efficiency and prediction capability of modern tree-ensemble algorithms in finding useful feature partitions for prediction from data. 
For the second step, we exploit the expressiveness of ASP in encoding constraints and preference to select useful rules from tree-ensembles, and rule selection is automated through a declarative encoding. 
In the end, we obtain the generated rule sets which serve as explanations for the tree-ensemble models, providing insights into their behavior. 
These aim to mimic the models' behavior rather than offering exhaustive and formally correct explanations, thus aligning with \textit{heuristic-based explanation methods} in the sense of e.g., \cite{izzaExplainingRandomForests2021,ignatievUsingMaxSATEfficient2022,audemardExplanatoryPowerBoolean2022}.

We then evaluate our approach using public datasets. For evaluating global explanations, we use the number and relevance of rules in the rule sets. The number of rules is often associated with explainability, with many rules being less desirable. Performance metrics such as classification accuracy, precision and recall can be used as a measure of relevance of the rules to the prediction task. For evaluating local explanations, we use precision and coverage metrics to compare against existing systems.

This paper makes the following contributions:
\begin{itemize}
    \item We present a novel application of Answer Set Programming (ASP) for explaining trained machine learning models. We propose a method to generate explainable rule sets from tree-ensemble models with ASP. More broadly, this work contributes to the growing body of knowledge on integrating symbolic reasoning with machine learning.
    \item We present how the rule set generation problem can be reformulated as an optimization problem, where we leverage existing knowledge on declarative pattern mining with ASP.
    \item We show how both global and local explanations can be generated by our approach, while comparative methods tend to focus on either one exclusively.
    \item To demonstrate the practical applicability of our approach, we provide both qualitative and quantitative results from evaluations with public datasets, where machine learning methods are used in a realistic setting.
\end{itemize}

The rest of this paper is organized as follows. 
In Section \ref{sec:background}, we review tree-ensembles, ASP and pattern mining. 
Section \ref{sec:rulesetgeneration} presents our method to generate rule sets from tree-ensembles using pattern mining and optimization encoded in ASP. 
Section \ref{sec:rulesetgenerationforglobal} describes global and local explanations in the context of our approach. 
Section \ref{sec:experiments} presents experimental results on public datasets. 
In Section \ref{sec:relatedwork} we review and discuss related works. 
Finally, in Section \ref{sec:conclusion} we present the conclusions.

\section{Background}\label{sec:background}

In the remainder of this paper, we shall use \textit{learning algorithms} to refer to methods used to train \textit{models}, as in machine learning literature. We use \textit{models} and \textit{explanations} to refer to machine learning models and post-hoc explanations about the said models, respectively. 

\subsection{Tree-Ensemble Learning Algorithms}

\textit{Tree-Ensemble (TE)} learning algorithms are machine learning methods widely used in practice, typically, when learning from tabular datasets. A trained TE model consists of multiple base decision trees, each trained on an independent subset of the input data. For example, Random Forests \cite{breimanRandomForests2001} and Gradient Boosted Decision Tree (GBDT) \cite{friedmanGreedyFunctionApproximation2001} are tree-ensemble learning algorithms. Recent surge of efficient and effective GBDT algorithms, e.g., LightGBM \cite{keLightGBMHighlyEfficient2017}, has led to wide adoption of TE learning algorithms in practice. Although individual decision trees are considered to be interpretable \cite{huysmansEmpiricalEvaluationComprehensibility2011}, ensembles of decision trees are seen as less interpretable.

The purpose of using TE learning algorithms is to train models that predict the unknown value of an attribute \(y\) in the dataset, referred to as \textit{labels}, using the known values of other attributes \(\mathbf{x}=(x_1,x_2,...,x_m)\), referred to as \(\textit{features}\). For brevity, we restrict our discussion to classification problems. During the training or learning phase, each input instance to the TE learning algorithm is a pair of features and labels, i.e. \((\mathbf{x}_i, y_i)\), where \(i\) denotes the instance index, and during the prediction phase, each input instance only includes features, \((\mathbf{x}_i)\), and the model is tasked to produce predictions \(\hat{y}_i\). A collection of input instances, complete with features and labels, is referred to as a \textit{dataset}. 
Given a dataset \(\mathcal{D}=\{(\mathbf{x}_i, y_i)\}\) with \(n\in\mathbb{N}\) examples and \(m\in\mathbb{N}\) features, a decision tree classifier \(t\) will predict the class label \(\hat{y}_i\) based on the feature vector \(\mathbf{x}_i\) of the \(i\)-th sample: \(\hat{y}_i = t(\mathbf{x}_i)\). A tree-ensemble \(\mathcal{T}\) uses \(K\in\mathbb{N}\) trees and additionally an aggregation function \(f\) over the \(K\) trees which combines the output from the trees: \(\hat{y}_i = f(t_{k\in K}(\mathbf{x}_i))\). As for Random Forest, for example, \(f\) is a majority voting scheme (i.e. \texttt{argmax} of \texttt{sum}), and in GBDT \(f\) may be a summation followed by softmax to obtain \(\hat{y}_i\) in terms of probabilities.

In this paper, a decision tree is assumed to be a binary tree where the internal nodes hold split conditions (e.g., \(x_1 \leq 0.5\)) and leaf nodes hold information related to class labels, such as the number of supporting data points per class label that have been assigned to the leaf nodes.
Richer collections of decision trees provide higher performance and less uncertainty in prediction compared to a single decision tree. Typically, each TE model has specific algorithms for learning base decision trees, adding more trees and combining outputs from the base trees to produce the final prediction. In GBDT, the base trees are trained sequentially by fitting the residual errors from the previous step. Interested readers are referred to \cite{friedmanGreedyFunctionApproximation2001}, and its more recent implementations, LightGBM \cite{keLightGBMHighlyEfficient2017} and XGBoost \cite{chenXGBoostScalableTree2016}.

\subsection{Answer Set Programming}
\textit{Answer Set Programming} \cite{lifschitzWhatAnswerSet2008} has its roots in logic programming and non-monotonic reasoning. A {\it normal logic program} is a set of rules of the form
\[\mathrm{a_1} \ \text{:-} \ \ \mathrm{a_2},\ \dots, \ \mathrm{a_m}, \ \mathrm{not} \ \mathrm{a_{m+1}}, \ \dots, \ \mathrm{not} \  \mathrm{a_n}.\]
where each \(\mathrm{a_i}\) is a first-order atom with \(1 \leq \mathrm{i} \leq \mathrm{n}\) and \texttt{not} is {\it default negation}. If only \(\mathrm{a_1}\) is included (\(\mathrm{n} = 1\)), the above rule is called a {\it fact}, whereas if \(\mathrm{a_1}\) is omitted, it represents an {\it integrity constraint}. A normal logic program induces a collection of intended interpretations, which are called \textit{answer sets}, defined by the stable model semantics \cite{gelfondStableModelSemantics1988}. 
Additionally, in modern ASP systems, constructs such as {\it conditional literals} and {\it cardinality constraints} are supported. The former in \textit{clingo} \cite{DBLP:journals/corr/GebserKKS14} is written in the form \(\{ \mathrm{a}(\texttt{X}) \ \text{:} \ \mathrm{b}(\texttt{X}) \}\),\footnote{Unless otherwise noted, we follow the Prolog-style notation in logic programs where strings beginning with a capital letter are variables, and others are predicate symbols or constants.} and expanded into the conjunction of all instances of \(\mathrm{a}(\texttt{X})\) where corresponding \(\mathrm{b}(\texttt{X})\) holds. 
The latter are written in the form \(s_1 \ \{ \mathrm{a}(\texttt{X}) \ \text{:} \ \mathrm{b}(\texttt{X}) \} \ s_2 \), which is interpreted as \(s_1 \leq \texttt{\#count} \{ \mathrm{a}(\texttt{X}) \ \text{:} \ \mathrm{b}(\texttt{X}) \} \leq s_2 \) where \(s_1\) and \(s_2\) are treated as lower and upper bounds, respectively; thus the statement holds when the count of instances \(\mathrm{a}(\texttt{X})\) where \(\mathrm{b}(\texttt{X})\) holds, is between \(s_1\) and \(s_2\). The minimization (or maximization) of an objective function can be expressed with \(\texttt{\#minimize}\) (or  \(\texttt{\#maximize}\)). 
Similarly to the \texttt{\#count} aggregate, the \texttt{\#sum} aggregate sums the first element (weight) of the terms, while also following the set property.
\textit{clingo} supports multiple optimization statements in a single program, and one can implement multi-objective optimization with priorities by defining two or more optimization statements.
For more details on the language of \textit{clingo}, we refer the reader to the \textit{clingo} manual.\footnote{\url{https://github.com/potassco/guide/releases/}}

\subsection{Pattern Mining}
In a general setting, the goal of pattern mining is to find interesting patterns from data, where patterns can be, for example, itemsets, sequences and graphs. For example, in \textit{frequent itemset mining} \cite{agrawalFastAlgorithmsMining1994}, the task is to find all subsets of items that occur together more than the threshold count in databases.
In this work, a \textit{pattern} is a set of predictive rules. 
A \textit{predictive rule} has the form \( c \Leftarrow s_1 \wedge s_2 \wedge , ..., s_n \), where \(c\) is a class label, and \(\{s_i\}\) (\(1 \leq i \leq n\)) represents conditions. 

For pattern mining with constraints, the notion of \textit{dominance} is important, which intuitively reflects the pairwise preference relation \((<^*)\) between patterns \cite{negrevergneDominanceProgrammingItemset2013}. Let \(C\) be a constraint function that maps a pattern to \(\{\top, \bot\}\), and let \(p\) be a pattern, then the pattern \(p\) is \textit{valid} iff  \(C(p)=\top\), otherwise it is \textit{invalid}. An example of \(C\) is a function which checks that the support of a pattern is above the threshold. The pattern \(p\) is said to be \textit{dominated} iff there exists a pattern \(q\) such that \(p <^* q\) and \(q\) is valid under \(C\). Dominance relations have been used in ASP encoding for pattern mining \cite{paramonovHybridASPbasedApproach2019}.

%Furthermore, \(p\) is said to be \textit{condensed} under constraints \(C\) iff it is valid and not dominated under \(C\). In this work we consider \textit{skyline} patterns \cite{ugarteSkypatternMiningPattern2017}, which in frequent itemset mining can be defined as \(p <^* q\) holds iff a) \(sup(p) \leq sup(q)\) and \(size(p) < size(q)\) or b) \(sup(p) < sup(q)\) and \(size(p) \leq size(q)\)). Skyline patterns are extracted according to Pareto principle that combines several quality measures. We will later show that this concept can be extended to more than two criteria.

\section{Rule Set Generation}\label{sec:rulesetgeneration}

\subsection{Problem Statement}\label{sec:problemstatement}
The rule set generation problem is represented as a tuple \(P=\{R,M,C,O\}\), where \(R\) is the set of all rules extracted from the tree-ensemble, \(M\) is the set of meta-data and properties associated with each rule in \(R\), \(C\) is the set of user-defined constraints including preferences, and \(O\) is the set of optimization objectives. 
The goal is to generate a set of rules from \(R\) by selection under constraints \(C\) and optimization objectives \(O\), where constraints and optimization may refer to the meta-data \(M\). In the following sections, we describe how we construct each \(R\), \(M\), \(C\) and \(O\), and finally, how we solve this problem with ASP.

\subsection{Rule Extraction from Decision Trees}\label{sec:ruleextraction}

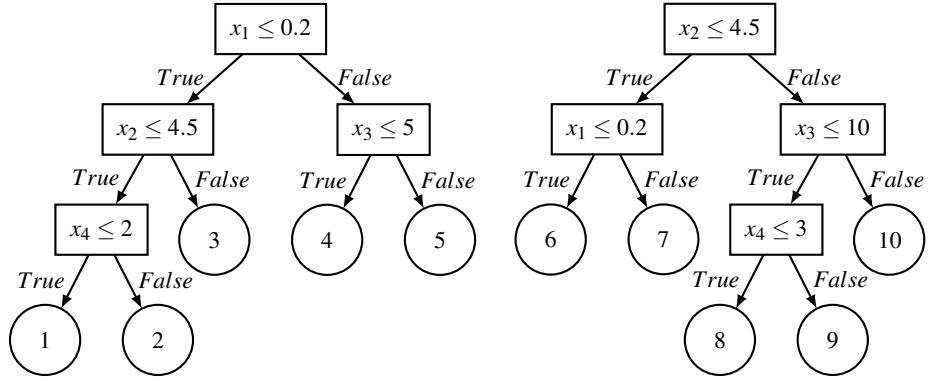
\begin{figure}[tb]
  \centering 
    % \usetikzlibrary{shapes.geometric}
% \usetikzlibrary{arrows,backgrounds}
% \usetikzlibrary{fit}
% \usetikzlibrary{positioning}

\begin{tikzpicture}

\tikzstyle{level 1}=[level distance=1cm, sibling distance=3cm]
\tikzstyle{level 2}=[level distance=1cm, sibling distance=1.5cm]
\tikzstyle{level 3}=[level distance=1cm, sibling distance=1.5cm]
\tikzstyle{level 4}=[level distance=1cm, sibling distance=1cm]

\tikzset{decision/.style={
draw=black, solid, rectangle, 
anchor=north,
inner sep=2mm, outer sep=0, text centered, growth parent anchor=south} 
}
\tikzset{prediction/.style={
draw=black, solid, circle, 
anchor=north,
text width=.3cm,
inner sep=2mm, outer sep=0, text centered, growth parent anchor=south} 
}
\tikzset{arrow/.style={-latex, thick}}
\tikzset{arrow-dashed/.style={-latex, dashed, thick}}

\node [decision] (t1) {$ x_1 \leq 0.2 $} [solid, -latex, thick]
  child { 
    node [decision] {$ x_2 \leq 4.5$}
      child {
        node [decision] {$x_4 \leq 2$} [solid, -latex, thick]
          child {
            node [prediction] {1} [solid, -latex, thick]
          edge from parent node [left,above,xshift=-0.45cm,yshift=-0.2cm] {$ True $}}
          child {
            node [prediction] {2} [solid, -latex, thick]
          edge from parent node [right,above,xshift=0.5cm,yshift=-0.2cm] {$ False $}}
      edge from parent node [left,above,xshift=-0.45cm,yshift=-0.2cm] {$ True $} }
      child {
        node [prediction] {3} [solid, -latex, thick]
      edge from parent node [right,above,xshift=+0.5cm,yshift=-0.2cm] {$ False $}}
  edge from parent node [left,above,xshift=-0.45cm,yshift=-0.2cm] {$ True $} } 
  child {
    node [decision] {$x_3 \leq 5$}
      child {
        node [prediction] {4}
      edge from parent node [left,above,xshift=-0.45cm,yshift=-0.2cm] {$ True $}}
      child {
        node [prediction] {5} [solid, -latex, thick]
      edge from parent node [right,above,xshift=+0.5cm,yshift=-0.2cm] {$ False $}}
  edge from parent node [right,above,xshift=+0.5cm,yshift=-0.2cm] {$ False $}}
;

\node [decision, right=4.5cm of t1] {$ x_2 \leq 4.5 $} [solid, -latex, thick]
  child { 
    node [decision] {$ x_1 \leq 0.2$}
      child {
        node [prediction] {6}
        edge from parent node [right,above,xshift=-0.45cm,yshift=-0.2cm] {$ True $}
        }
      child {
        node [prediction] {7} [solid, -latex, thick]
      edge from parent node [right,above,xshift=+0.5cm,yshift=-0.2cm] {$ False $}}
  edge from parent node [left,above,xshift=-0.45cm,yshift=-0.2cm] {$ True $} } 
  child {
    node [decision] {$x_3 \leq 10$}
        child {
        node [decision] {$x_4 \leq 3$} [solid, -latex, thick]
          child {
            node [prediction] {8} [solid, -latex, thick]
          edge from parent node [left,above,xshift=-0.45cm,yshift=-0.2cm] {$ True $}}
          child {
            node [prediction] {9} [solid, -latex, thick]
            edge from parent node [right,above,xshift=0.5cm,yshift=-0.2cm] {$ False $}}
          edge from parent node [left,above,xshift=-0.45cm,yshift=-0.2cm] {$ True $}}
      child {
        node [prediction] {10} [solid, -latex, thick]
      edge from parent node [right,above,xshift=+0.5cm,yshift=-0.2cm] {$ False $}}
  edge from parent node [right,above,xshift=+0.5cm,yshift=-0.2cm] {$ False $}}
;
\end{tikzpicture}
    % \vspace{-.5\baselineskip}
  \caption{A simple decision tree-ensemble consisting of two decision trees. The rule associated with each node is given by the conjunction of all conditions associated with nodes on the paths from the root node to that node.}
  \label{fig:simpletree}
%   \vspace{-.5\baselineskip}
\end{figure}

This subsection describes how \(R\), the set of all rules, is constructed. The first two steps in ``tree-ensemble processing" in Figure \ref{fig:pipeline} are also described in this subsection.
Recall that a tree-ensemble \(\mathcal{T}\) is a collection of \(K\) decision trees, and we refer to individual trees \(t_k\) with subscript \(k\). An example of a decision tree-ensemble is shown in Figure \ref{fig:simpletree}. A decision tree \(t_k\) has \(N_{t_k}\) nodes and \(L_{t_k}\) leaves. Each node represents a split condition, and there are \(L_{t_k}\) paths from the root node to the leaves. For simplicity, we assume only features that have orderable values (continuous features) are present in the dataset in the examples below.\footnote{Real datasets may have unorderable categorical values. For example, in the \textit{census} dataset, occupation (Sales, etc.) and education (Bachelors, etc.) are categorical features. Support for categorical feature split is implementation-dependent, however in general one can replace the continuous split with a subset selection, e.g., \(x_c \in \{x_{c1}, x_{c2},...\}\).} The tree on the left in Figure \ref{fig:simpletree} has 4 internal nodes including the root node with condition \([x_1 \leq 0.2]\) and 5 leaf nodes; therefore there are 5 paths from the root note to the leaf nodes 1 to 5. 

From the left-most path of the decision tree on the left in Figure \ref{fig:simpletree}, the following prediction rule is created. We assume that node 1 predicts class label 1 in this instance.\footnote{Label=1 and 0 refer to the attributes in the dataset and have different meaning depending on the dataset. For example, in the \textit{census} dataset, label=1 and 0 mean that the personal income is more than \$50,000 and that it is no more than \$50,000, respectively. In a binary classification setting, the task would be to predict whether the personal income is greater than \$50,000 (\(>\)\$50,000).}
\begin{equation} \label{eqn:ex_class1}
    class(1) \Leftarrow (x_1 \leq 0.2) \wedge (x_2 \leq 4.5) \wedge (x_4 \leq 2)
\end{equation}
Assuming that node 2 predicts class label 0, we also construct the following rule (note the reversal of the condition on \(x_4\)):
\begin{equation} \label{eqn:ex_class2}
    class(0) \Leftarrow (x_1 \leq 0.2) \wedge (x_2 \leq 4.5) \wedge (x_4 > 2)
\end{equation}
% We can also construct subsets of rules by applying each of the conditions sequentially and computing the predicted label at each step. For example, from the last rule, we may construct the following rule:
% \begin{equation*}
%     class(1) \Leftarrow (x_1 \leq 0.2) \wedge (x_2 \leq 4.5)
% \end{equation*}

\begin{algorithm}[tb]
\caption{Construct candidate rule set \(R\)}
\label{alg:extract_rules}
\begin{algorithmic}[1] %[1] enables line numbers

% Enumerate trees
\Require Tree-ensemble of \(K\) trees, \(\mathcal{D}\) dataset 
\Ensure \(R\) candidate rule set 
\Function {DecomposeTreeEnsemble}{tree-ensemble}
\State \(R\) $\gets$ \(\varnothing\)
\For {tree in tree-ensemble}
    \State \(R\) $\gets$ \(R \; \cup\) \Call{ExtractRulesFromTree}{tree, $ \mathcal{D} $ }
\EndFor
\State \Return \(R\)
\EndFunction
\item[]

% No internal nodes, just end leaf nodes
% Enumerate paths and rules
\Require \(k\)-th tree in the tree-ensemble, \(\mathcal{D}\) dataset
\Ensure \(R_k\) candidate rule set from the \(k\)-th tree 
\Function {ExtractRulesFromTree}{tree}
\State \(R_k\) $\gets$ \(\varnothing\)
\State paths $\gets$ Enumerate all paths to the leaf nodes
\For {path in paths}
    \State \(B_k\) $\gets$ \(\varnothing\) \Comment{Body of the rule, which is a set of split conditions}
    \State \textit{class} $\gets$ null
    \For {node in path}
        \If {node is left child}
            \State \(B_k\) $\gets$ \(B_k \; \cup \) (feature $ \leq $ threshold)
        \Else
            \State \(B_k\) $\gets$ \(B_k \; \cup \) (feature $ > $ threshold)
        \EndIf
    \EndFor
    \State \textit{class} $\gets$ Apply \(B_k\) to \(\mathcal{D}\), and choose the class with the largest instance count
    \State \textit{rule} $\gets$ \textit{class} $ \Leftarrow  B_k $ \Comment{Construct a rule, as in Section \ref{sec:ruleextraction}}
    \State \(R_k\) $\gets$  \(R_k \; \cup \) \textit{rule} \Comment{Add to the set for \(k\)-th tree}
\EndFor
\State \Return \(R_k\)
\EndFunction

\end{algorithmic}
\end{algorithm}

To obtain the candidate rule set, we essentially decompose a tree-ensemble into a rule set. The steps are outlined in Algorithm \ref{alg:extract_rules}.
% The set of all rules, \(R\), is constructed as follows: 
% \begin{enumerate}
%     \item Enumerate all possible paths from the root node to the leaves. %For a binary decision tree with depth \(d_k\), the maximum number of leaf nodes is \(2^{d_k}\), which is also the maximum number of paths from the root node to the leaf nodes.
%     \item For each path, at each subsequent node on the path to the leaf node, the split condition of the node is appended to the body (antecedent, set of conditions) of the rule. %For a decision tree, the maximum number of such rules is the same as the maximum number of nodes in the tree, i.e., \(2^{d_k + 1}-1\).
%     \item Compute the predicted class label for each rule. For simplicity, we apply all conditions in the rule and calculate the most likely class label from the count data (\(\texttt{argmax}\) of counts).
%     \item Add the generated rules to the candidate rule set \(R\).
%     \item Repeat steps 1 to 4 for each tree \(t_k\) where \(1 \leq k \leq K\), in the ensemble of \(K\) trees.
% \end{enumerate}
By constructing the candidate rule set \(R\) in this way, the bodies (antecedents) of rules included in rule sets are guaranteed to exist in at least one of the trees in the tree-ensemble. 
Rule sets generated in this manner are therefore faithful to the representation of the original model in this sense. 
If we were to construct rules from the unique set of split conditions, the resulting rule may have combinations of conditions that do not exist in any of the trees.

% \begin{proposition}
% If \(R\) is constructed according to the steps 1-5, then the bodies of the rules in \(R\) exist in at least one of the trees in the tree-ensemble.
% \end{proposition}
% \begin{proof}
% Suppose there is a rule in \(R\) whose body does not exist in any of the trees. Steps 1 enumerates all possible paths from the root to the leaf nodes, and Step 2 follows the paths while constructing the bodies of the rules. A non-existent path is excluded from Step 1. Therefore, if a path does not exist in at least one of the trees in the ensemble, a rule whose body is constructed from a non-existent path cannot be included in R. A contradiction.
% \end{proof}

%The maximum size of \(R\), assuming all rules are unique, is \(\sum_{k=1}^{K}2^{d_k + 1}-1\), however, we may choose to use only the rules that contain the leaf nodes\footnote{While leaf nodes themselves do not contain split conditions, the overall effect is the same as counting the number of leaf nodes because we add both outcomes (\textit{True} or \textit{False}) of the last seen split condition.}, and reduce the size of \(R\) to \(\sum_{k=1}^{K}2^{d_k}\). In practice, there are duplicate split conditions across trees in a tree-ensemble, so the unique count of rules is often smaller than the maximum value.

We now analyze the computational complexities associated with constructing the set of all rules \(R\).
Let us assume that (1) all \(K\) trees in the ensemble are perfect binary decision trees and have the same height \(h\), (2) there are \(n\) examples and \(m\) features in the dataset, and (3) there are no duplicate rules and conditions across trees.
% \begin{proposition}\label{prop:candidateruleset}
% The maximum size of \(R\) is \(K \times (2^{h+1} - 2)\). The reduced size, constructed by only considering the rules at the leaf nodes, is \(K \times 2^h\).
% \end{proposition}
% \begin{proof}
% The reduced size (\(K \times 2^h\)) follows immediately from the number of leaf nodes in a perfect binary decision tree with height \(h\), i.e., \(2^h\).
% The number of internal nodes is \(2^h-1\), and by enumerating possible outcomes (\textit{True} or \textit{False}) of internal nodes we have \(2(2^{h} - 1)=2^{h+1}-2\) rules from each tree.
% \end{proof}
\begin{proposition}\label{prop:candidateruleset}
The maximum size of \(R\), constructed by only considering the rules at the leaf nodes, is \(K \times 2^h\).
\end{proposition}
This follows immediately from the number of leaf nodes in a perfect binary decision tree with height \(h\), i.e., \(2^h\).
In practice, there are duplicate split conditions across trees in a tree-ensemble, so the unique count of rules is often smaller than the maximum value.
% \begin{proposition}
% The space complexity of the proposed method to construct \(R\) is \(O(K \times 2^{d+1} + n \times m)\) for all rules, and \(O(K \times 2^d + n \times m)\) for rules containing leaf nodes only.
% \end{proposition}
% \begin{proof}
% For a tabular dataset in a dense matrix format, the storage requirement is \(O(n \times m)\). The number of rules for each tree is \(O(2^{d+1})\) and \(O(2^{d})\) for all rules and reduced rules, respectively.
% \end{proof}
% \begin{proposition}
% The time complexity of the proposed method to construct \(R\) is \(O(K \times (2^{h+1} \times n \times h))\) for all rules, and \(O(K \times (2^h \times n \times h))\) for the reduced size case.
% \end{proposition}
% \begin{proof}
% %Enumerating paths to leaf nodes takes \(O(2^{h+1})\) time by depth first search. 
% There are \(O(2^{h+1})\) rules for the all rules case, and \(O(2^h)\) rules for the reduced size case. For each rule, all conditions in a rule need to be applied to the data. Since there are at most \(h\) conditions in a rule, and there are \(n\) examples, it takes \(O(n \times h)\) time to apply all conditions in a rule.
% \end{proof}
\begin{proposition}
The time complexity of the proposed method to construct \(R\) is \(O(K \times (2^h \times n \times h))\).
\end{proposition}
\begin{proof}
%Enumerating paths to leaf nodes takes \(O(2^{h+1})\) time by depth first search. 
For each rule in \(O(2^h)\) rules, all conditions in the rule need to be applied to the data. Since there are at most \(h\) conditions in a rule, and there are \(n\) examples, it takes \(O(n \times h)\) time to apply all conditions in a rule.
\end{proof}
% As shown in these propositions, we can save computation time and space by using the reduced size extraction method, \textcolor{red}{in exchange for excluding rules derived from partial paths that do not include leaf nodes.}

% % \noindent
% \begin{figure}[tb]
%   \centering 
% %   \sffamily
%     \input{simple_tree.tikz}
%   \caption{A decision tree. The rule associated with each node is given by conjunction of all conditions associated with nodes on the paths from the root node to that node.}
%   \label{fig:simpletree}
% \end{figure}

\subsection{Computing Metrics and Meta-data for Selection}\label{sec:metrics}
After the candidate rule set \(R\) is constructed, we gather information about the performance and properties of each rule and collect them into a set \(M\).
This is the last step in the tree-ensemble processing process depicted in Figure \ref{fig:pipeline} (``Assign Metrics").
The meta-data, or properties, of a rule are information such as the size of the rule, as defined by the number of conditions in the rule, and the ratio of instances which are covered by the rule.
Computing classification metrics, e.g., accuracy and precision, requires access to a subset of the dataset with ground truth labels, which could be either a training or a validation set.
On the other hand, when access to the labeled subset is not available at runtime, these metrics and their corresponding predicates cannot be used in the ASP encoding. 
In our experiments, we used the training sets to compute these classification metrics during rule set generation, and later used the validation sets to evaluate their performance.

Performance metrics measure how well a rule can predict class labels. Here we calculate the following performance metrics: accuracy, precision, recall and F1-score, as shown below. 
\begin{align}
    \textit{accuracy} &= \frac{TP + TN}{TP + TN + FP + FN} \qquad & \textit{precision} &= \frac{TP}{TP+FP} \nonumber \\
    \textit{recall}   &= \frac{TP}{TP+FN}                   \qquad & \textit{F1-score} &= 2 \times \frac{precision \times recall}{precision + recall}
\end{align}
For classification tasks, a \textit{true positive (TP)} and a \textit{true negative (TN)} refer to a correctly predicted positive class and negative class, respectively, with respect to the labels in the dataset.
Conversely, a \textit{false positive (FP)} and a \textit{false negative (FN)} refer to an incorrectly predicted positive class and negative class, respectively.
These metrics are not specific to the rules, and can be computed for trained tree-ensemble models, as well as explanations of trained machine learning models, as we shall show later in Section \ref{sec:fidelityexperiments}.
We compute multiple metrics for a single rule, to meet a range of user requirements for explanation. 
One user may only be interested in simply the most accurate rules (maximize accuracy), whereas another user could be interested in more precise rules (maximize precision), or rules with more balanced performance (maximize F1-score). 

%The meta-data, or properties, of a rule are information such as the size of the rule, as defined by the number of conditions in the rule, or the number of instances which are covered by the rule. %These properties are used in the selection step, and they can also be used to define competing objectives.
%For example, one can expect a very long rule with large number of conditions to be precise, but the rule may be too specific and may not cover many instances. %Moreover, a long rule is more difficult to comprehend than a short, concise rule. In this case, the size property needs to be minimized, while the precision metric is maximized.

\begin{table}
    \caption{List of predicates representing a rule in ASP.}
    \label{tab:asp_predicates}
    % \centering
    \begin{minipage}{\textwidth}
    \begin{tabular}{ll}%{p{3.1cm}p{4.4cm}}
    \hline\hline
    Predicate & Meaning\footnote{Properties and metrics marked with asterisks(*) are multiplied by 100 and rounded to the nearest integer.} \\
    \hline
    \small{\texttt{rule(X)}}              & \texttt{X} holds the rule index. \\
    \small{\texttt{condition(X,I)}}       & Rule \texttt{X} has condition \texttt{I}. \\
    \small{\texttt{size(X,L)}}            & Number of conditions in rule \texttt{X} (length, \texttt{L}). \\
    \small{\texttt{predict\_class(X,C)}}  & Predicted class label \texttt{C} of rule \texttt{X}.\\
    \small{\texttt{support(X,S)}}         & Support \texttt{S} of rule \texttt{X}, the ratio of instances that are covered by rule \texttt{X}.*\\
    \small{\texttt{error\_rate(X,E)}}     & Error rate (\(1 - accuracy\)), \texttt{E}, of the rule \texttt{X} evaluated in the training data.*\\
    \small{\texttt{accuracy(X,A)}}        & Accuracy score \texttt{A} of rule \texttt{X}.*\\
    \small{\texttt{precision(X,P)}}       & Precision score \texttt{P}  of rule \texttt{X}.*\\
    \small{\texttt{recall(X,R)}}          & Recall score \texttt{R}  of rule \texttt{X}.*\\
    \small{\texttt{f1\_score(X,F)}}       & F1-score \texttt{F} of rule \texttt{X}.*\\
    \hline\hline
    \end{tabular}
    \vspace{-2\baselineskip}
    \end{minipage}
\end{table}

The candidate rule set \(R\) and meta-data set \(M\) are represented as facts in ASP, as shown in Table \ref{tab:asp_predicates}. 
For example, Rule \ref{eqn:ex_class1} (the first rule in Section \ref{sec:ruleextraction}) may be represented as follows:\footnote{The performance metrics are for illustration purposes only and are chosen arbitrarily.}
\begin{Verbatim}[frame=single,fontsize=\small]
% rule 1
rule(1). condition(1,1). condition(1,2). condition(1,3). support(1,10).
size(1,3). accuracy(1,50). error_rate(1,50). precision(1,30). 
recall(1,40). f1_score(1,34). predict_class(1,1).
\end{Verbatim}
Unique conditions are indexed and denoted by the condition predicate. For instance, in the example above (representing Rule \ref{eqn:ex_class1}), ``\texttt{condition(1,1)}" represents \((x_1 \leq 0.2)\), ``\texttt{condition(1,2)}" corresponds to \((x_2 \leq 4.5)\), and so forth.

\subsection{Encoding Inclusion Criteria and Constraints}\label{sec:encodingconstraints}
As with previous works in pattern mining in ASP, we follow the ``generate-and-test" approach, where a set of potential solutions are generated by a choice rule and subsequently constraints are used to filter out unacceptable candidates. In the context of rule set generation, we use a choice rule to generate candidate rule sets that may constitute a solution (``Generate Candidate Rule Sets" in Figure \ref{fig:pipeline}).
In this section, we introduce the following selection criteria and constraints: (1) individual rule selection criteria that are applied on a per-rule basis, (2) pairwise constraints that are applied to pairs of rules, and (3) collective constraints that are applied to a set of rules.

The ``generator" choice rule has the following form:
\begin{Verbatim}[frame=single,fontsize=\small]
% pick at least 1 rule and at maximum B rules for each class K.
1 { selected(X) :  predict_class(X, K), valid(X) } B :- class(K).
\end{Verbatim}

\begin{example}\label{ex:generator}
\begin{Verbatim}[frame=single,fontsize=\small]
% pick at least 1 rule and at maximum 5 rules for each class K.
1 { selected(X) :  predict_class(X, K), valid(X) } 5 :- class(K).
\end{Verbatim}
\end{example}
The choice rule above generates candidate subsets of size between 1 and 5 from \(R\), where we use the \texttt{selected/1} predicate to indicate that a rule (\texttt{rule(X)}) is included in the subset. 

Individual rule selection criteria are integrated into the generator choice rule by the \texttt{valid/1} predicate, where a rule \texttt{rule(X)} is valid whenever \texttt{invalid(X)} cannot be inferred.
\begin{Verbatim}[frame=single,fontsize=\small]
valid(X) :- rule(X), not invalid(X).
\end{Verbatim}

\begin{example}\label{ex:invalid}
The following criterion excludes rules with low support from the candidate set:
\begin{Verbatim}[frame=single,fontsize=\small]
% this will exclude rules that apply to less than 5% of instances
invalid(X) :- rule(X), support(X,S), S < 5.
\end{Verbatim}
\end{example}

Pairwise constraints can be used to encode dominance relations between rules. For a rule \texttt{X} to be dominated by \texttt{Y}, \texttt{Y} must be strictly better in one criterion than \texttt{X} and at least as good as \texttt{X} or better in other criteria. In the following case, we encode the dominance relation between rules using the accuracy metric and support, where we prefer rules that are accurate and cover more data.
\begin{Verbatim}[frame=single,fontsize=\small]
% cannot be dominated
:- dominated.
% X is dominated by Y
gt_acc_geq_cov(Y) :- selected(X), valid(Y),
    accuracy(X,Ax), accuracy(Y,Ay), support(X,Spx), support(Y,Spy),
    Ax < Ay, Spx <= Spy.
geq_acc_gt_cov(Y) :- selected(X), valid(Y),
    accuracy(X,Ax), accuracy(Y,Ay), support(X,Spx), support(Y,Spy),
    Ax <= Ay, Spx < Spy.
dominated :- valid(Y), gt_acc_geq_cov(Y).
dominated :- valid(Y), geq_acc_gt_cov(Y).
\end{Verbatim}

Collective constraints are applied to collections of rules, as opposed to individual or pairs of rules. 
The following restricts the maximum number of conditions in rule sets, using the \texttt{\#sum} aggregate:\footnote{Since the \texttt{\#sum} aggregate sums the first element of the terms (\texttt{S} in this case) while also following the set property, adding \texttt{X} to the term tuple is needed to allow the same weight \texttt{S} to be added more than once.}
\begin{Verbatim}[frame=single,fontsize=\small]
% total number of conditions should not exceed 30
:- #sum { S,X : size(X,S), selected(X) } > 30.
\end{Verbatim}

We envision two main use-cases for the criteria and constraints introduced in this section: (1) to generate rule sets with certain properties, and (2) to reduce the computation time. For (1), the user can use the individual selection criteria to ensure that the rules included into the candidate rule sets have certain properties, or the collective constraints to put restrictions on the aggregate properties of the rule sets. The latter use-case has more practical relevance because in our case, as in pattern mining, the complexity of a naive ``generate-and-test" approach is exponential with respect to the number of candidates. 

To reduce the search space, one can place an upper bound on the size of generated candidate sets, and use the \texttt{invalid/1} predicate to prevent unacceptable rules being included into the candidates, as shown above. Because setting unreasonable conditions leads to zero rule sets generated, care should be taken when using the selection criteria and constraints for this purpose. 
In particular, if any of the metric predicates listed in Table \ref{tab:asp_predicates} are used in defining \texttt{invalid/1}, e.g., \texttt{invalid(X) :- rule(X), metric(X, N), N < B.}, to avoid all \texttt{rule(X)} being \texttt{invalid(X)}, one should respect the conditions listed in Table \ref{tab:metric_bounds}.
In the following example, we will show how the \texttt{invalid/1} predicate can be used to reduce the search space.
\begin{example}\label{ex:reducesearch}
Let the logic program be:
\begin{Verbatim}[frame=single,fontsize=\small,numbers=left,xleftmargin=5mm]
1 { selected(X) : predict_class(X, K), valid(X) } 1 :- class(K).
valid(X) :- rule(X), not invalid(X).
invalid(X) :- rule(X), metric(X, N), N < B.
\end{Verbatim}
Then, there is at least one valid rule if \(\texttt{B} \leq max(\texttt{N}_1,...,\texttt{N}_{|R|})\).
Let \(\texttt{B} = 1+max(\texttt{N}_1,...,\texttt{N}_{|R|})\), then by line 3 (\(\texttt{N} < \texttt{B}\)), all rules will be \texttt{invalid}, and \texttt{valid(X)} cannot be inferred. Then, the choice rule (line 1) is not satisfied. Alternatively, let \(\texttt{B} = max(\texttt{N}_1,...,\texttt{N}_{|R|})\), then there is at least one rule such that \(\texttt{N} = \texttt{B}\). Since \(\texttt{invalid(X)}\) cannot be inferred for such a rule, it is \(\texttt{valid}\) and the choice rule is satisfied.
\end{example}

The upper bound parameter (5 in Example \ref{ex:generator}, and 1 in Example \ref{ex:reducesearch}) controls the potential maximum number of rules that can be included in a rule set. 
The actual number of rules that emerge in the final rule sets is highly dependent on the selection criteria, user preferences, and the characteristics of the tree-ensemble model. 
Practically, we recommend initially setting this parameter to a lower value (e.g., 3) while focusing on refining other aspects of the encoding, since this allows for a more manageable starting point. 
Given the ``generate-and-test" approach, high values may lead to excessively slow run time.
If it becomes evident that a larger rule set could be beneficial, the parameter can be incrementally increased. 
This ensures more efficient use of computational resources while also catering to the evolving needs of the encoding process.

\begin{table}
    \caption{List of minimum and maximum values for the bounds used in defining \(\texttt{invalid/1}\).}
    \label{tab:metric_bounds}
    % \centering
    \begin{minipage}{\textwidth}
    \begin{tabular}{lcll}%{p{3.1cm}p{4.4cm}}
    \hline\hline
    Metric Predicate                    & Relation          & Condition & Intention \\
    \hline
    \small{\texttt{size(X,L)}}          & \texttt{L > B}    & \(\texttt{B} \geq min\{\texttt{L}_1,...,\texttt{L}_{|R|}\}\) & Invalid if the rule is too long  \\
    \small{\texttt{support(X,S)}}       & \texttt{S < B}    & \(\texttt{B} \leq max\{\texttt{S}_1,...,\texttt{S}_{|R|}\}\) & Invalid if the rule has low support  \\
    \small{\texttt{error\_rate(X,E)}}   & \texttt{E > B}    & \(\texttt{B} \geq min\{\texttt{E}_1,...,\texttt{E}_{|R|}\}\) & Invalid if the rule has high error rate \\
    \small{\texttt{accuracy(X,A)}}      & \texttt{A < B}     & \(\texttt{B} \leq max\{\texttt{A}_1,...,\texttt{A}_{|R|}\}\) & Invalid if the rule has low accuracy \\
    \small{\texttt{precision(X,P)}}     & \texttt{P < B}    & \(\texttt{B} \leq max\{\texttt{P}_1,...,\texttt{P}_{|R|}\}\) & Invalid if the rule has low precision \\
    \small{\texttt{recall(X,R)}}        & \texttt{R < B}    & \(\texttt{B} \leq max\{\texttt{R}_1,...,\texttt{R}_{|R|}\}\) & Invalid if the rule has low recall \\
    \small{\texttt{f1\_score(X,F)}}     & \texttt{F < B}    & \(\texttt{B} \leq max\{\texttt{F}_1,...,\texttt{F}_{|R|}\}\) & Invalid if the rule has low F1-score \\
    \hline\hline
    \end{tabular}
    %\vspace{-2\baselineskip}
    \end{minipage}
\end{table}

\subsection{Optimizing Rule Sets}\label{sec:optimizingrulesets}
Finally, we pose the rule set generation problem as a multi-objective optimization problem, given the aforementioned facts and constraints encoded in ASP. The desiderata for generated rule sets may contain multiple competing objectives. For instance, we consider a case where the user wishes to collect accurate rules that cover many instances, while minimizing the number of conditions in the set. This is encoded as a group of optimization statements:
\begin{Verbatim}[frame=single,fontsize=\small]
% maximize accuracy and support, minimize the number of conditions
#maximize { A,X : selected(X), accuracy(X,A)}.
#maximize { S,X : selected(X), support(X,S)}.
#minimize { L,X : selected(X), size(X,L)}.
\end{Verbatim}

Instead of maximizing/minimizing the sums of metrics, we may wish to optimize more nuanced metrics, such as average accuracy and coverage of selected rules:
\begin{Verbatim}[frame=single,fontsize=\small]
% maximize average accuracy and coverage
selected_rules(SR) :- SR = #count { I : selected(I) }, SR != 0.
#maximize { Ai/(S*SR)@3,I : selected(I), size(I,S), 
            accuracy(I,Ai), selected_rules(SR) }.
#maximize { Sp/S@2,I : selected(I), size(I,S), support(I,Sp) }.
\end{Verbatim}
This metric can be maximized by selecting the smallest number of short and accurate rules. Similar metrics can be defined for precision-coverage,
\begin{Verbatim}[frame=single,fontsize=\small]
% maximize average precision and coverage
#maximize { Pi/(S*SR)@3,I : selected(I), size(I,S), 
            precision(I,Pi), selected_rules(SR) }.
#maximize { Sp/S@2,I : selected(I), size(I,S), support(I,Sp) }.
\end{Verbatim}
and for precision-recall.
\begin{Verbatim}[frame=single,fontsize=\small]
% maximize average precision and recall
#maximize { Pi/(S*SR)@3,I : selected(I), size(I,S), 
            precision(I,Pi), selected_rules(SR) }.
#maximize { R/S@2,I : selected(I), size(I,S), recall(I,R) }.
\end{Verbatim}

For optimization, we introduce a measure of overlap between the rules to be minimized. Intuitively, minimizing this objective should result in rule sets where rules share only a few conditions, which should further improve the explainability of the resulting rule sets. Specifically, we introduce a predicate \texttt{rule\_overlap(X,Y,Cn)} to measure the degree of overlap between rules \texttt{X} and \texttt{Y}.
\begin{Verbatim}[frame=single,fontsize=\small]
% number of shared conditions between rules
rule_overlap(X,Y,Cn) :- selected(X), selected(Y), X!=Y,
    Cn = #count { Cx : Cx=Cy, condition(X,Cx), condition(Y,Cy) }.
#minimize { Cn,X : selected(X), selected(Y), rule_overlap(X,Y,Cn) }.
\end{Verbatim}

\section{Rule Set Generation for Global and Local Explanations}\label{sec:rulesetgenerationforglobal}
In this section, we will describe how to generate global and local explanations with the rule set generation method.  \citeN{guidottiSurveyMethodsExplaining2019} defined global explanation as descriptions of how the overall system works, and local explanation as specific descriptions of why a certain decision was made. We shall now adopt these definitions to our rule set generation task from tree-ensemble models.
\begin{definition}
A \textit{global explanation} is a set of rules derived from the tree-ensemble model, that approximates the overall predictive behavior of the base tree-ensemble model. 
\end{definition}
Examples of measures of approximation for global explanations are: accuracy, precision, recall, F1-score, fidelity and coverage.
\begin{example}
Given a tree-ensemble as in Figure \ref{fig:simpletree}, a global explanation can be constructed from a candidate rule set that includes all possible paths to the leaf nodes (1, 2, ..., 10 in Figure \ref{fig:simpletree}), then selecting rules based on user-defined criteria.
\end{example}
\begin{definition}
An instance to be explained is called a \textit{prediction instance}.
A \textit{local explanation} is a set of rules derived from the tree-ensemble model, that approximates the predictive behavior of the base tree-ensemble model when applied to a specific prediction instance. 
\end{definition}

\begin{example}
Given a tree-ensemble as in Figure \ref{fig:simpletree}, and a prediction instance, a local explanation can be constructed by only considering rules that were active during the prediction, then selecting rules based on user-defined criteria. For example, if leaf nodes 2 and 6 were active, then \(R\) only includes rules constructed from the paths leading to nodes 2 and 6. 
Here, a leaf node is considered \textit{active} during prediction if the decision path for the prediction instance leads to it, meaning the conditions leading up to that node are satisfied by the instance's features.
\end{example}
The predictive behavior in this context refers to the method by which the model makes the prediction (aggregating decision tree outputs) and the outcomes of the prediction. 
The differences between the global and local explanations have implications on the encoding we use for rule set generation. 
Note that these two types of explanations serve distinct purposes. 
The global explanation seeks to explain the model's overall behavior, while the local explanation focuses on the reasoning behind a specific prediction instance. 
There is no inherent expectation for a global explanation to align with or fully encompass a local explanation.
In particular, when a local explanation is applicable to multiple instances due to these instances having similar feature values, for instance, this local explanation might not be able to accurately predict for these instances. 
This is measured by a precision metric, and evaluated further in Section \ref{sec:localexp} 

In Table \ref{tab:exampleexplanation} we show examples of global and local explanations on the same dataset (\textit{adult}).
For this dataset, the task is to predict whether an individual earns more or less than \$50,000 annually. 
The global and local explanations consist of 4 conditions, and share an attribute (hours-per-week) with different threshold values.
While these two rules have the same outcome, the attributes in the bodies are different: in this instance, the global explanation focuses more on the numerical attributes, while the local explanation contains categorical attributes.

% Global
% class 1 IF age > 27.5 AND capital-gain > 7055.5 AND age <= 60.5 AND hours-per-week <= 89.5
% Prediction 1
% Explanation
% hours-per-week <= 89.5 AND
% age > 27.5 AND
% age <= 60.5 AND
% captal-gain > 7055.5 

% Local
% Prediction 1
% Explanation
% hours-per-week > 30.5 AND
% relationship in {Husband, Wife} AND
% education in {Assoc-acdm, Assoc-voc, Bachelors, Doctorate, Masters, Prof-school, Some-college} AND
% occupation in {Exec-managerial, Prof-specialty, Protective-serv, Sales, Tech-support}

\begin{table}
    % \centering
    \caption{Examples of global and local explanations from the adult dataset (LightGBM + ASP).}
    \label{tab:exampleexplanation}
    \begin{minipage}{\textwidth}
    \begin{tabular}{p{0.07\textwidth} p{0.15\textwidth} p{0.75\textwidth}}
    \hline\hline
        Type & Class & Explanation \\
    \hline
    Global & Income \(>\) 50k & 
    hours-per-week \(\leq\) 89.5 AND age \(>\) 27.5 AND age \(\leq\) 60.5 \newline 
    AND captal-gain \(>\) 7055.5 \\
    & & \\
    Local & Income \(>\) 50k & 
    hours-per-week \(>\) 30.5 AND relationship in \(\{\)Husband, Wife\(\}\) \newline 
    AND education in \(\{\)Assoc-acdm, Assoc-voc, Bachelors, Doctorate, \newline Masters, Prof-school, Some-college\(\}\) \newline 
    AND occupation in \(\{\)Exec-managerial, Prof-specialty, Protective-serv, \newline Sales, Tech-support\(\}\) \\

    \hline\hline
    \end{tabular}
    %\vspace{-2\baselineskip}
    \end{minipage}
\end{table}

Recall that we start with the candidate rule set, \(R\), which is created by processing the tree-ensemble model. The rules in \(R\) are different between global and local explanations, even when the underlying tree-ensemble model is the same. For global explanations, we can enumerate all rules including internal nodes (Section \ref{sec:ruleextraction}) regardless of the outcomes of the rules because we are more interested in obtaining a simpler classifier with the help of constraints (Section \ref{sec:encodingconstraints}) and optimization criteria (Section \ref{sec:optimizingrulesets}). On the other hand, for local explanations, it is necessary to consider the match between the rules' prediction and the actual outcome of the tree-ensemble model as to keep the precision of explanations high. 

By definition, a local explanation should describe the behavior of the model \textit{on a single prediction instance}. Thus, we shall make the following modifications to \(R\) when generating rule sets for local explanations. 
We start from the candidate rule set \(R\) as in Algorithm \ref{alg:extract_rules} (Section \ref{sec:ruleextraction}), then, for each predicted instance:
\begin{enumerate}
    \item Identify the leaf nodes that were active during the prediction.
    \item Exclude rules that did not participate in the prediction.
    \item Replace the outcome of the rule with the predicted label.
\end{enumerate}
After the modification as outlined above, the maximum size of the starting rule set  will be the number of trees in a tree-ensemble. 
Let \(K\) be the number of decision trees in a tree-ensemble model, then, since there is exactly one leaf node per tree responsible for the prediction, there will be \(K\) rules.
Compared to the global explanation case (Proposition \ref{prop:candidateruleset}), the size of the candidate rule set is exponentially smaller for the local explanation.
This reduction is enabled by analyzing the behavior of the decision trees during prediction, and it is one of the benefits of using an explanation method which can take advantage of the structure of the model under study.

\section{Experiments}\label{sec:experiments}
In this section, we present a comprehensive evaluation of our rule set generation framework, focusing on both global and local explanations. We evaluate the explanations on public datasets using various metrics, and we also compare the performance to existing methods, including rule-based classifiers.
We used several metrics to assess the quality of the generated explanations. 
These metrics are designed to evaluate different aspects of the explanations, including their comprehensibility, fidelity and usability. 
Below, we provide an overview of the metrics used in our evaluations. 
Detailed discussions of these metrics can be found in the respective sections of this paper.

\textbf{Global Explanations} (Section \ref{sec:globalexp}): 
\begin{itemize}
    \item \textbf{Number of Rules and Conditions} (Section \ref{sec:globalnumrules} and \ref{sec:global_num_literals}): Assesses the simplification of the original model by counting the number of rules and conditions.
    \item \textbf{Relevance} (Section \ref{sec:relevanceofrules}): Measures the relevance of the rules by comparing the classification performance against the original model.
    \item \textbf{Fidelity} (Section \ref{sec:fidelityexperiments}): Measures the degree to which the rules accurately describe the behavior of the original model.
    \item \textbf{Run Time} (Section \ref{sec:globalruntime}): Measures the efficiency of generating global explanations.
\end{itemize}

\textbf{Local Explanations} (Section \ref{sec:localexp}):
\begin{itemize}
    \item \textbf{Number of Conditions} (Section \ref{sec:localnumcondition}): Measures the conciseness of the local explanations by counting the number of conditions.
    \item \textbf{Local-Precision and Coverage} (Section \ref{sec:localprecision}): Local-precision compares the model's predictions for instances covered by the local explanation with the prediction for the instance that induced the explanation. Local-coverage measures the proportion of instances in the validation set that are covered by the local explanation.
    \item \textbf{Run Time} (Section \ref{sec:localruntime}): Measures the efficiency of generating local explanations.
\end{itemize}

\subsection{Experimental Setup}
\subsubsection{Datasets}
We used in total 14 publicly available datasets, where except for the \textit{adult}\footnote{\url{https://github.com/propublica/compas-analysis}} dataset, all datasets were taken from the UCI Machine Learning Repository\footnote {\url{https://archive.ics.uci.edu/ml/index.php}}\cite{Dua:2019}.
Datasets were chosen from public repositories to ensure a diverse range in terms of the number of instances, the number of categorical variables, and class balance.
This was done intentionally, to observe the variance in, for example, explanation generation times. 
Additionally, the variation in categorical variables ratio and class balance was designed to produce a wide array of tree ensemble configurations (e.g., more or fewer trees, varying widths and depths).
We expected these configurations, in turn, to influence the nature of the explanations generated.
We included 3 datasets ({\it adult}, {\it credit german}, {\it compas}) for comparison because they were widely used in local explainability literature. 
The \textit{adult} dataset is actually a subset of the \textit{census} dataset, but we included the former for consistency with existing literature, and the latter for demonstrating the applicability of our approach to larger datasets. 
The summary of these datasets is shown in Table \ref{tab:datadescription}.

\subsubsection{Experimental Settings}
We used \textit{clingo} 5.4.0\footnote{\url{https://potassco.org/clingo/}} \cite{DBLP:journals/corr/GebserKKS14} for answer set programming, and set the time-out to 1,200 seconds. 
We used RIPPER implemented in Weka \cite{wittenWEKAWorkbenchOnline2016} and an open-source implementation of RuleFit\footnote{\url{https://github.com/christophM/rulefit}} where Random Forest was selected as the rule generator, and scikit-learn\footnote{\url{https://scikit-learn.org/}} \cite{scikit-learn} for general machine learning functionalities. 
Our experimental environment is a desktop machine with Ubuntu 18.04, Intel Core i9-9900K 3.6GHz (8 cores/16 threads) and 64 GB RAM. 
For reproducibility, all source codes for the implementation, experiments and preprocessed datasets are available from our GitHub repository.\footnote{\url{https://github.com/atakemura/treetap}}

Unless noted otherwise, all experimental results reported here were obtained with 5-fold cross validation, with hyperparameter optimization in each fold. 
% To obtain results in a reasonable amount of time, we used the reduced size rule extraction method, where only complete rules leading to the leaf nodes were considered, and the internal nodes were ignored. 
To evaluate the performance of the extracted rule sets, we implemented a naive rule-based classifier, which is constructed from the rule sets extracted with our method. In this classifier, we apply the rules sequentially to the validation dataset and if all conditions within a rule are true for an instance in the dataset, the consequent of the rule is returned as the predicted class. More formally, given a set of rules \(R_s \subset R\) with cardinality \(|R_s|\) that shares the same consequent \(class(Q)\), we represent this rule-based classifier as the disjunction of antecedents of the rules:
\begin{equation*}
    class(Q) \Leftarrow body(R_1) \vee body(R_2) \vee ... \vee body(R_r) \textrm{ where } 1 \leq r \leq |R_s|
\end{equation*}
For a given data point, it is possible that there are no rules applicable, and in such cases the most common class label in the training dataset is returned.

\begin{landscape}
\begin{table}
    \caption{Dataset description, selected hyperparameters and candidate rule counts.}
    \label{tab:datadescription}
    \begin{minipage}{\linewidth}
    \begin{tabular}{l rrr p{0.13\textwidth} rr rrr rrr}
    \hline\hline
       &  & & & & \multicolumn{2}{c}{Decision Tree} & \multicolumn{3}{c}{Random Forest} & \multicolumn{3}{c}{LightGBM} \\
                    & & & & &{MaxD\footnote{The maximum depth parameter. The minimum and maximum values set during the hyperparameter search are shown in parentheses. The hyperparameters shown in this table are averaged over 5 folds.}} & {\(|R|\)}\footnote{The number of candidate rules, averaged over 5 folds.} & {\#Tree\footnote{The number of trees (estimators) parameter.}} & {MaxD} & {\(|R|\)} & {\#Tree} & {MaxD} & {\(|R|\)} \\ \cmidrule(lr){6-7} \cmidrule(lr){8-10} \cmidrule(lr){11-13}
    Dataset         & \#Instance & \#Feature\footnote{The number of features (columns). The number of categorical features is shown in parentheses.} & \#Ratio & Meaning of y = 1 & (2, 9) & {} & (50, 500) & (2, 9) & {} & {(30, 1,000)} & (2, 9) & {} \\
    \hline
    adult           & 48,842    & 12 (8)    & 0.24    & income \(>\) 50k    & 9.0 & 104.6 & 240 & 9.0 & 8,180.0 & 206.2 & 5.0 & 4,227.6 \\
    autism          & 704       & 20 (18)   & 0.27    & screening result    & 5.0 & 2.0   & 200 & 5.8 & 2,083.4 & 485.6 & 6.6 & 2.0 \\
    breast          & 699       & 9 (9)     & 0.34    & malignant           & 6.4 & 16.8  & 200 & 5.0 & 1,391.2 & 131.2 & 6.4 & 194.0 \\
    cars            & 1,728     & 6 (6)     & 0.30    & acceptable condition& 9.0 & 41.4  & 396 & 8.4 & 13,027.4 & 889.0 & 7.4 & 1,308.0 \\
    census          & 299,285   & 40 (33)   & 0.06    & income \(>\) 50k    & 8.2 & 81.8  & 320 & 9.0 & 9,773.4 & 198.2 & 9.0 & 9,533.0 \\
    compas          & 7,214     & 11 (7)    & 0.28    & 2 year recidivism   & 8.2 & 62.8  & 320 & 9.0 & 23,254.4 & 108.4 & 6.2 & 985.4 \\
    credit australia & 690      & 14 (8)    & 0.44    & application accepted& 9.0 & 3.0   & 212 & 7.0 & 3,178.2 & 68.4 & 5.2 & 528.0 \\
    credit german   & 1,000     & 20 (13)   & 0.30    & good creditor       & 8.4 & 32.4  & 280 & 9.0 & 13,488.4 & 55.2 & 5.0 & 364.6 \\
    credit taiwan   & 30,000    & 23 (10)   & 0.22    & payment next month  & 7.0 & 53.6  & 280 & 9.0 & 15,976.4 & 138.8 & 7.0 & 4,150.0 \\
    heart           & 270       & 13 (8)    & 0.44    & disease present     & 7.6 & 12.4  & 288 & 6.6 & 2,906.6 & 32.8 & 7.4 & 298.4 \\
    ionosphere      & 351       & 34 (0)    & 0.64    & good radar return   & 4.8 & 9.2   & 324 & 5.4 & 1,317.2 & 48.0 & 5.8 & 311.8 \\
    kidney          & 400       & 24 (13)   & 0.62    & chronic disease     & 5.4 & 6.8   & 200 & 5.0 & 996.4 & 817.8 & 8.2 & 359.0 \\
    krvskp          & 3,196     & 36 (36)   & 0.52    & white can win       & 9.0 & 33.6  & 320 & 9.0 & 9,647.4 & 293.0 & 8.2 & 2,354.2 \\
    voting          & 435       & 16 (16)   & 0.61    & democrat            & 5.8 & 10.2  & 240 & 6.6 & 2,555.0 & 44.6 & 5.8 & 168.0 \\

\hline\hline
\end{tabular}
\vspace{-2\baselineskip}
\end{minipage}
\end{table}
\end{landscape}

\subsection{Evaluating Global Explanations}\label{sec:globalexp}
Let us recall that the purpose of generating global explanations is to provide the user with a simpler model of the original complex model. Thus, we introduce proxy measures to evaluate (1) the degree to which the model is simplified, by the number of extracted rules and conditions,  (2) the relevance of the extracted rules, by comparing classification performance metrics against the original model, and (3) the degree to which the explanation accurately describes the behavior of the original model, by fidelity metrics.

We conducted the experiment in the following order. 
First, we trained Decision Tree, Random Forest and LightGBM on the datasets in Table \ref{tab:datadescription}. 
Selected optimized hyperparameters of the tree-ensemble models are also reported in Table \ref{tab:datadescription}. 
Further details on hyperparameter optimization are available in \ref{sec:appendix_hyperparameter}.
We then applied our rule set generation method to the trained tree-ensemble models. 
Finally, we constructed a naive rule-based classifier using the set of rules extracted in the previous step, and calculated performance metrics on the validation set. 
This process was repeated in a 5-fold stratified cross validation setting to estimate the performance. 
We compare the characteristics of our approach against the known methods RIPPER and RuleFit.

We used the following selection criteria to filter out rules that were considered to be undesirable; for example, those rules with low accuracy or low coverage. We used the same set of selection criteria for all datasets, irrespective of underlying label distribution or learning algorithms. When the candidate rules violate any one of those criteria, they are excluded from the candidate rule set, which means that in the worst case where all the candidate rules violate at least one criterion, this encoding will result in an empty rule set (see Section \ref{sec:encodingconstraints}).
\begin{Verbatim}[frame=single,fontsize=\small]
% exclude long rules
invalid(I) :- size(I,S), S > 10, rule(I).
% exclude inaccurate rules
invalid(I) :- error_rate(I,E), E > 70, rule(I).
% exclude low precision rules
invalid(I) :- precision(I,P), P < 2, rule(I).
% exclude low recall rules
invalid(I) :- recall(I,R), R < 2, rule(I).
% exclude low coverage rules
invalid(I) :- support(I,Sp), Sp < 2, rule(I).
\end{Verbatim}
Another scenario in which our method will produce an empty rule set is when the tree-ensemble contains only ``leaf-only" or ``stump" trees, that have one leaf node and no splits. In this case, we have no split information to create candidate rules; thus, an empty rule set is returned to the user. This is often caused by inadequate setting of hyperparameters that control the growth of the trees, especially when using imbalanced datasets. It is however outside the scope of this paper, and we will simply note such cases (empty rule set returned) in our results without further consideration. 

\subsubsection{Number of Rules}\label{sec:globalnumrules}

The average sizes of generated rule sets are shown in Table \ref{tab:numrulecond}. 
The sizes of candidate rule sets, from which the rule sets are generated, are listed in the \(|R|\) columns in Table \ref{tab:datadescription}.
Rule set size of 1 means that the rule set contains a single rule only.
As one might expect, the Decision Tree consistently has the smallest candidate rule set, but in some cases the Random Forest produced considerably more candidate rules than the LightGBM, e.g., \textit{cars, compas}. 
Our method can produce rule sets which are significantly smaller than the original model, based on the comparison between the sizes of the candidate rule set \(|R|\) and resulting rule sets.

We will now compare our method to the two benchmark methods, RuleFit and RIPPER. 
The average size of generated rule sets is shown in Table \ref{tab:numrulecond}. 
RuleFit includes original features (called linear terms) as well as conditions extracted from the tree-ensembles in the construction of a sparse linear model, that is to say, the counts in Table \ref{tab:numrulecond} may be inflated by the linear terms. 
On the other hand, the output from RIPPER only contains rules, and RIPPER has rule pruning and rule set optimization to further reduce the rule set size. Moreover, RIPPER has direct control over which conditions to include into rules, whereas our method and RuleFit rely on the structure of the underlying decision trees to construct candidate rules. 
Our method consistently produced smaller rule sets compared to RuleFit and RIPPER, although the difference between our method and RIPPER was not as pronounced when compared to the difference between our method and RuleFit. RuleFit produced the largest number of rules compared with other methods, although they were much smaller than the original Random Forest models (Table \ref{tab:numrulecond}).

\begin{landscape}
\begin{table}
    \caption{Size of rule sets, total and average number of conditions in rules}
    \label{tab:numrulecond}
    \begin{minipage}{\linewidth}
\begin{tabular}{lrrrrrrrrrrrrrrr}
\hline\hline
{} & \multicolumn{5}{c}{Size of Rule Sets} & \multicolumn{5}{c}{Total number of conditions in rules} & \multicolumn{5}{c}{Average number of conditions in rules} \\\cmidrule(lr){2-6} \cmidrule(lr){7-11} \cmidrule(lr){12-16}
Dataset &     DT\footnote{Decision Tree + ASP} & RF \footnote{Random Forest + ASP} & LGBM\footnote{LightGBM + ASP} & RuleFit & RIPPER &        DT & RF & LGBM & RuleFit & RIPPER &        DT & RF & LGBM & RuleFit & RIPPER \\
\hline
adult            &         1.0 &     1.4 &       1.4 &     358.0 &     25.0 &            4.0 &     5.8 &       7.0 &    1142.6 &    110.8 &            4.0 &     3.6 &       5.0 &       3.2 &      4.4 \\
autism           &         1.0 &     1.0 &       1.0 &       3.2 &      2.0 &            1.0 &     1.0 &       1.0 &       4.4 &      1.0 &            1.0 &     1.0 &       1.0 &       1.3 &      0.5 \\
breast           &         1.0 &     1.2 &       1.0 &     508.0 &     14.8 &            5.8 &     6.0 &       1.0 &    1443.6 &     19.8 &            5.8 &     5.0 &       1.0 &       3.0 &      1.3 \\
cars             &         1.0 &     1.0 &       1.0 &     478.0 &     28.0 &            6.4 &     5.4 &       1.0 &    1536.8 &    101.8 &            6.4 &     5.4 &       1.0 &       3.3 &      3.4 \\
census           &         1.0 &     1.0 &       1.0 &     736.0 &     60.6 &            8.2 &     3.4 &       1.8 &    1967.4 &    381.2 &            8.2 &     3.4 &       1.8 &       2.7 &      6.3 \\
compas           &         1.0 &     1.0 &       1.0 &     174.8 &     10.4 &            3.8 &     2.0 &       2.4 &     554.4 &     27.6 &            3.8 &     2.0 &       2.4 &       3.2 &      2.6 \\
credit australia &         1.0 &     1.0 &       1.0 &      50.6 &      5.4 &            2.0 &     2.4 &       3.2 &     134.4 &      9.8 &            2.0 &     2.4 &       3.2 &       2.6 &      1.6 \\
credit german    &         1.2 &     2.0 &       1.2 &     392.8 &      4.2 &            4.6 &    18.0 &       3.0 &    1137.8 &      8.4 &            3.6 &     9.0 &       2.0 &       3.0 &      2.0 \\
credit taiwan    &         1.0 &     1.8 &       2.6 &     110.2 &      7.2 &            5.2 &    13.4 &       7.6 &     311.2 &     18.0 &            5.2 &     6.2 &       3.4 &       2.8 &      2.5 \\
heart            &         1.0 &     1.0 &       1.6 &     424.8 &      5.6 &            2.4 &     3.0 &       4.8 &    1127.0 &     13.4 &            2.4 &     3.0 &       3.0 &       2.3 &      2.1 \\
ionosphere       &         1.0 &     1.0 &       1.0 &     398.8 &      5.0 &            4.6 &     4.8 &       5.4 &     973.4 &      6.2 &            4.6 &     4.8 &       5.4 &       2.5 &      1.1 \\
kidney           &         1.0 &     1.0 &       1.2 &      94.6 &      4.6 &            1.6 &     2.4 &       4.2 &     259.2 &      7.2 &            1.6 &     2.4 &       3.6 &       2.7 &      1.5 \\
krvskp           &         1.0 &     1.0 &       1.0 &     276.0 &     16.4 &            6.0 &     7.8 &       3.0 &     881.2 &     52.8 &            6.0 &     7.8 &       3.0 &       3.2 &      3.2 \\
voting           &         1.0 &     1.0 &       1.0 &     300.6 &      3.4 &            2.2 &     1.0 &       2.0 &     803.6 &      5.8 &            2.2 &     1.0 &       2.0 &       2.6 &      1.3 \\
\hline\hline
\end{tabular}
\end{minipage}
\end{table}
\end{landscape}

\subsubsection{Number of Conditions in Rules}\label{sec:global_num_literals}
In this subsection, we compare the average number of conditions in each rule and the total number of conditions in rules.
One would expect a more precise rule to have a larger number of conditions in its body compared to the one that is more general.
It should be noted that, however, due to the experimental condition, the maximum number of conditions in a single rule is set by the maximum depth parameter in each of the learning algorithms, which in turn is set by the hyperparameter tuning algorithm.

The average number of conditions in each rule are shown in Table \ref{tab:numrulecond}.
We note that sometimes the algorithms may produce rules without any conditions in the bodies, such as when the induced trees have only a single split node at the root; thus the average number reported in Table \ref{tab:numrulecond} may be biased towards lower numbers.
From the table, we see that the average number of conditions in a rule generally falls in the range of between 1 and 10, and this is consistent with the search range of hyperparameters we set for the experiments.
Table \ref{tab:numrulecond} shows the total number of conditions in a rule set.
Unlike the average number of conditions in a rule, we see a large difference between our method and the benchmark methods.
In all datasets, RuleFit produced the highest counts of conditions in rules, followed by RIPPER and the ensemble-based methods.
From Table \ref{tab:numrulecond}, we make the following observations: (1) the length of individual rules does not vary as much as the number of rules between different methods (2) the high number of conditions in rules extracted by RuleFit can be explained by the high number of rules, where the length of individual rules are comparable to other methods.

\subsubsection{Relevance of Rules}\label{sec:relevanceofrules}

To quantify the relevance of the extracted rules, we measured the ratio of performance metrics using the naive rule-based classifier by 5-fold cross validation (Table \ref{tab:performanceratio}). A performance ratio of less than 1.0 means that the rule-based classifier performed worse than the original classifier (LightGBM and Random Forest), whereas a performance ratio greater than 1.0 means the rule set's performance is better than the original classifier. We used a version of the ASP encoding shown in Section \ref{sec:optimizingrulesets} where the accuracy and coverage are maximized. RIPPER was excluded from this comparison because it has a built-in rule generation and refinement process, and it does not have a base model, whereas our method and RuleFit use variants of tree-ensemble models as base models.

From Table \ref{tab:performanceratio} we observe that in terms of accuracy, RuleFit generally performs as well as, or marginally better than, the original Random Forest. 
On the other hand, although our method can produce rule sets that are comparable in performance against the original model, they do not produce rules that perform significantly better. 
With Decision Tree and Random Forest, the generated rule sets perform much worse than the original model, e.g., in \textit{kidney}, \textit{voting}. 
The LightGBM+ASP combination resulted in the second-best performance overall, where the resulting rules' performances were arguably comparable (0.8-0.9 range) to the original model with a few exceptions (e.g., \textit{census} F1-score) where the performance ratio was about half of the original. 
While RuleFit's performance was superior, our method could still produce rule sets with reasonable performance with much smaller rule sets that are an order of magnitude smaller than RuleFit. 
A rather unexpected result was that using our method (Random Forest) or RuleFit significantly improved the F1-score in the \textit{census} dataset. 
In Table \ref{tab:performanceratio} we can see that recall was the major contributor to this improvement.

\begin{landscape}
\begin{table}
    \caption{Average ratio of rule-based classifiers' performance vs. original tree-ensembles, averaged over 5 folds. (Global Explanations)}
    \label{tab:performanceratio}
    \begin{minipage}{\linewidth}
\begin{tabular}{l rrrr rrrr rrrr rrrr}
\hline\hline
{} & \multicolumn{4}{c}{Accuracy ratio} & \multicolumn{4}{c}{F1 ratio} & \multicolumn{4}{c}{Precision ratio} & \multicolumn{4}{c}{Recall ratio} \\\cmidrule(lr){2-5}\cmidrule(lr){6-9}\cmidrule(lr){10-13}\cmidrule(lr){14-17}
Dataset & DT\footnote{Decision Tree + ASP} & RF\footnote{Random Forest + ASP} & LGBM\footnote{LightGBM + ASP} & RuleFit & DT & RF & LGBM & RuleFit & DT & RF & LGBM & RuleFit & DT & RF & LGBM & RuleFit \\ 
\hline
adult            &         0.92 &         0.91 &     \textbf{0.94} &    1.01 &         0.34 &         0.63 &     0.78 &    \textbf{1.12} &         \textbf{1.30} &         0.95 &     0.86 &    0.94 &         0.22 &         0.75 &     0.74 &    \textbf{1.25} \\
autism           &         \textbf{1.00} &         \textbf{1.00} &     \textbf{1.00} &    \textbf{1.00} &         1.00 &         \textbf{1.01} &     1.00 &    \textbf{1.01} &         \textbf{1.00} &         \textbf{1.00} &     \textbf{1.00} &    \textbf{1.00} &         1.00 &         \textbf{1.01} &     1.00 &    \textbf{1.01} \\
breast           &         0.91 &         0.96 &     0.95 &    \textbf{0.98} &         0.80 &         0.93 &     0.91 &    \textbf{0.97} &         0.97 &         0.97 &     0.96 &   \textbf{0.99} &         0.70 &         0.91 &     0.87 &    \textbf{0.95} \\
cars             &         0.80 &         0.81 &     0.52 &    \textbf{1.01} &         0.41 &         0.50 &     0.44 &    \textbf{1.02} &         1.04 &         \textbf{1.05} &     0.36 &    1.04 &         0.26 &         0.32 &     0.60 &    \textbf{1.00} \\
census           &         0.97 &         0.99 &     0.80 &    \textbf{1.02} &         0.16 &         5.63 &     0.40 &   \textbf{10.87} &         0.18 &         0.44 &     0.30 &    \textbf{0.76} &         0.16 &         8.38 &     1.26 &   \textbf{17.14} \\
compas           &         0.95 &         0.85 &     0.85 &    \textbf{1.01} &         0.44 &         0.84 &     0.77 &    \textbf{1.08} &         \textbf{1.12} &         0.63 &     0.73 &    0.96 &         0.32 &         1.13 &     0.96 &    \textbf{1.18} \\
credit australia &         0.89 &         0.89 &     0.88 &    \textbf{1.00} &         0.77 &         0.82 &     0.77 &    \textbf{1.00} &         \textbf{1.17} &         1.00 &     1.04 &    1.00 &         0.55 &         0.73 &     0.67 &    \textbf{0.99} \\
credit german    &         0.91 &         0.92 &     0.88 &    \textbf{0.95} &         0.46 &         0.83 &     0.90 &    \textbf{1.18} &         \textbf{1.01} &         0.69 &     0.72 &    0.78 &         0.39 &         0.94 &     1.13 &    \textbf{1.56} \\
credit taiwan    &         0.96 &         0.93 &     0.99 &    \textbf{1.01} &         0.21 &         0.63 &     0.83 &    \textbf{1.11} &         0.90 &         0.74 &     \textbf{1.05} &    0.98 &         0.15 &         0.70 &     0.75 &    \textbf{1.18} \\
heart            &         0.93 &         0.90 &     0.95 &    \textbf{1.00} &         0.78 &         0.82 &     0.80 &    \textbf{1.00} &         1.11 &         0.96 &     \textbf{1.22} &    1.00 &         0.69 &         0.73 &     0.63 &    \textbf{1.00} \\
ionosphere       &         0.73 &         0.69 &     0.98 &    \textbf{0.99} &         0.86 &         0.82 &     0.98 &    \textbf{0.99} &         0.71 &         0.70 &     \textbf{1.00} &    0.99 &         \textbf{1.10} &         1.02 &     0.96 &    0.99 \\
kidney           &         0.66 &         0.62 &     0.91 &    \textbf{1.00} &         0.80 &         0.77 &     0.93 &    \textbf{1.00} &         0.66 &         0.62 &     0.92 &    \textbf{1.00} &         \textbf{1.04} &         1.00 &     0.94 &    1.00 \\
krvskp           &         0.53 &         0.54 &     0.59 &    \textbf{1.02} &         0.69 &         0.71 &     0.67 &    \textbf{1.02} &         0.53 &         0.54 &     0.65 &    \textbf{1.03} &         1.00 &         \textbf{1.02} &     0.76 &    1.01 \\
voting           &         0.64 &         0.64 &     0.96 &    \textbf{0.99} &         0.78 &         0.79 &     0.96 &    \textbf{0.99} &         0.62 &         0.63 &     \textbf{1.01} &    0.98 &         1.04 &         \textbf{1.05} &     0.92 &    1.01 \\
\hline\hline
\end{tabular}
\vspace{-2\baselineskip}
\end{minipage}
\end{table}
\end{landscape}

\begin{landscape}
\begin{table}
    \caption{Fidelity metrics, averaged over 5 folds. (Global Explanations)}
    \label{tab:global_fidelity}
    \begin{minipage}{\linewidth}
\begin{tabular}{l rrrr rrrr rrrr rrrr}
\hline\hline
{} & \multicolumn{4}{c}{Accuracy} & \multicolumn{4}{c}{F1 score} & \multicolumn{4}{c}{Precision} & \multicolumn{4}{c}{Recall} \\\cmidrule(lr){2-5}\cmidrule(lr){6-9}\cmidrule(lr){10-13}\cmidrule(lr){14-17}
Dataset & DT\footnote{Decision Tree + ASP} & RF\footnote{Random Forest + ASP} & LGBM\footnote{LightGBM + ASP} & RuleFit & DT & RF & LGBM & RuleFit & DT & RF & LGBM & RuleFit & DT & RF & LGBM & RuleFit \\ 
\hline
adult            &         0.85 &         0.87 &     0.86 &    \textbf{0.94} &         0.29 &         0.42 &     0.44 &    \textbf{0.83} &         \textbf{1.00} &         0.87 &     0.93 &    0.73 &         0.17 &         0.33 &     0.35 &    \textbf{0.96} \\
autism           &         \textbf{1.00} &         \textbf{1.00} &     \textbf{1.00} &    \textbf{1.00} &         \textbf{1.00} &         0.99 &     \textbf{1.00} &    0.99 &         \textbf{1.00} &         0.99 &     \textbf{1.00} &    0.99 &         \textbf{1.00} &         \textbf{1.00} &     \textbf{1.00} &    \textbf{1.00} \\
breast           &         0.91 &         0.86 &     0.89 &    \textbf{0.97} &         0.82 &         0.75 &     0.81 &    \textbf{0.96} &         \textbf{1.00} &         0.97 &     0.96 &    0.98 &         0.72 &         0.64 &     0.71 &    \textbf{0.94} \\
cars             &         0.77 &         0.79 &     0.60 &    \textbf{0.98} &         0.39 &         0.48 &     0.52 &    \textbf{0.97} &         \textbf{1.00} &         \textbf{1.00} &     0.50 &    0.99 &         0.25 &         0.31 &     0.63 &    \textbf{0.95} \\
census           &         0.92 &         \textbf{0.97} &     0.90 &    0.96 &         0.00 &         0.07 &     \textbf{0.36} &    0.09 &         0.00 &         0.04 &     \textbf{0.56} &    0.05 &         0.00 &         0.61 &     0.43 &    \textbf{1.00} \\
compas           &         0.82 &         0.75 &     0.79 &    \textbf{0.94} &         0.38 &         0.50 &     0.50 &    \textbf{0.83} &         \textbf{0.93} &         0.41 &     0.68 &    0.76 &         0.26 &         0.71 &     0.53 &    \textbf{0.93} \\
credit australia &         0.72 &         0.81 &     0.78 &    \textbf{0.96} &         0.64 &         0.73 &     0.65 &    \textbf{0.95} &         \textbf{1.00} &         0.92 &     0.95 &    0.96 &         0.47 &         0.66 &     0.55 &    \textbf{0.95} \\
credit german    &         0.83 &         \textbf{0.85} &     0.75 &    0.81 &         0.48 &         0.28 &     0.43 &    \textbf{0.55} &         \textbf{0.93} &         0.66 &     0.54 &    0.43 &         0.42 &         0.25 &     0.50 &    \textbf{0.83} \\
credit taiwan    &         0.89 &         0.91 &     0.93 &    \textbf{0.98} &         0.20 &         0.46 &     0.55 &    \textbf{0.89} &         0.73 &         0.80 &     \textbf{0.98} &    0.82 &         0.12 &         0.36 &     0.44 &    \textbf{0.98} \\
heart            &         0.86 &         0.79 &     0.80 &    \textbf{0.89} &         0.74 &         0.66 &     0.67 &    \textbf{0.86} &         \textbf{1.00} &         0.93 &     \textbf{1.00} &    0.86 &         0.66 &         0.55 &     0.52 &    \textbf{0.86} \\
ionosphere       &         0.65 &         0.69 &     0.82 &    \textbf{0.95} &         0.79 &         0.81 &     0.84 &    \textbf{0.97} &         0.65 &         0.69 &     \textbf{0.99} &    0.97 &         \textbf{1.00} &         \textbf{1.00} &     0.75 &    0.96 \\
kidney           &         0.63 &         0.62 &     0.86 &    \textbf{1.00} &         0.77 &         0.77 &     0.88 &    \textbf{1.00} &         0.63 &         0.62 &     0.95 &    \textbf{1.00} &         \textbf{1.00} &         \textbf{1.00} &     0.83 &    1.00 \\
krvskp           &         0.53 &         0.53 &     0.62 &    \textbf{0.97} &         0.69 &         0.69 &     0.67 &    \textbf{0.97} &         0.53 &         0.53 &     0.72 &    \textbf{0.98} &         \textbf{1.00} &         \textbf{1.00} &     0.72 &    0.96 \\
voting           &         0.60 &         0.60 &     0.93 &    \textbf{0.97} &         0.75 &         0.75 &     0.94 &    \textbf{0.98} &         0.60 &         0.60 &     \textbf{1.00} &    0.96 &         \textbf{1.00} &         \textbf{1.00} &     0.88 &    1.00 \\
\hline\hline
\end{tabular}
\vspace{-2\baselineskip}
\end{minipage}
\end{table}
\end{landscape}

\subsubsection{Fidelity Metrics of Global Explanations}\label{sec:fidelityexperiments}

In Section \ref{sec:relevanceofrules}, we compared the ratio of performance metrics of different methods when measured against original labels. 
In the context of evaluating explanation methods, it is also important to investigate the \textit{fidelity}, i.e., to which extent the explanation is able to accurately imitate the original model \cite{guidottiSurveyMethodsExplaining2019}. 
A fidelity metric is calculated as an agreement metric between the prediction of the original model and the explanation, the latter in this case is the rule set. 
More concretely, when the predicted class by the model is positive and that by the explanation is positive, it is a true positive (TP), and when the latter is negative, it is a false negative (FN).
Thus, the fidelity metrics can be calculated in the same manner as performance metrics using the equations shown in Section \ref{sec:metrics}. 
RIPPER was excluded from this comparison for the same reasons as outlined in Section \ref{sec:relevanceofrules}.

The average fidelity metrics (accuracy, F1 score, precision and recall) are shown in Table \ref{tab:global_fidelity}.
The overall trend is similar to the previous section on rule relevance, where RuleFit performs the best overall in terms of fidelity.
The accuracy metrics for our method shows that the global explanations, in general, behaved similarly to the original model, although RuleFit was better in most of the datasets.
The precision metrics show that, even when excluding the results for Decision Tree (which is not a tree-ensemble learning algorithm), our method could produce explanations that had high fidelity in terms of precision compared to RuleFit.
The fidelity metrics may be improved further by including them in the ASP encodings, since they were not part of the selection criteria or optimization goals.

\subsubsection{Changing Optimization Criteria}

The definition of optimization objectives has a direct influence over the performance of the resulting rule sets, and the objectives need to be set in accordance with user requirements. The answer sets found by \textit{clingo} with multiple optimization statements are optimal regarding the set of goals defined by the user. Instead of using accuracy, one may use other rule metrics as defined in Table \ref{tab:asp_predicates} such as precision and/or recall. If there are priorities between optimization criteria, then one could use the priority notation (\texttt{weight@priority}) in \textit{clingo} to define them. Optimal answer sets can be computed in this way, however, if enumeration of such optimal sets is important, then one could use the \texttt{pareto} or \texttt{lexico} preference definitions provided by \textit{asprin} \cite{brewkaAsprinCustomizingAnswer2015} to enumerate Pareto optimal answer sets. Instead of presenting a single optimal rule set to the user, this will allow the user to explore other optimal rule sets.

To investigate the effect of changing optimization objectives, we changed the ASP encoding from max. accuracy-coverage to max. precision-coverage (shown in Section \ref{sec:encodingconstraints}) while keeping other parameters constant. 
The results are shown in Table \ref{tab:globalchangeenc}. 
Note that it is the ratio of precision score shown in the table, as opposed to accuracy or F1-score in the earlier tables. 
Here, since we are optimizing for better precision, we expect the precision-coverage encoding to produce rule sets with better precision scores than the accuracy-coverage encoding. 
For the Decision Tree and Random Forest + ASP, the effect was not as pronounced as we expected, but we observed noticeable differences in datasets \textit{compas} and \textit{credit german}. 
For the LightGBM+ASP combination, we observed more consistent difference, except for the \textit{credit german} dataset, the encoding produced intended results in most of the datasets in this experiment.

\begin{table}[]
    \centering
    \caption{Average ratio of rule-based classifier's precision vs. original tree-ensembles, averaged over 5 folds. (Global Explanations)}
    \label{tab:globalchangeenc}
    \begin{minipage}{\textwidth}
    \begin{tabular}{l rr rr rr }
    \hline\hline
        & \multicolumn{2}{c}{Decision Tree+ASP\footnote{Performance ratio of 1 means the rule set's precision is identical to the original classifier. Numbers are shown in bold where the performance ratio was better than more than 0.01 compared to the other encoding.}} & \multicolumn{2}{c}{Random Forest+ASP} & \multicolumn{2}{c}{LightGBM+ASP} \\ \cmidrule(lr){2-3} \cmidrule(lr){4-5} \cmidrule(lr){6-7}
Dataset  & acc.cov\footnote{acc.cov=accuracy and coverage encoding, see Section 4.} & prec.cov\footnote{prec.cov=precision and coverage encoding, see Section 4.} & acc.cov & prec.cov & acc.cov & prec.cov \\
    \hline

adult            &         1.30 &     1.30 &         0.95 &     \textbf{1.13} &     0.86 &     \textbf{1.27} \\
autism           &         1.00 &     1.00 &         1.00 &     1.00 &     1.00 &     1.00 \\
breast           &         0.97 &     0.97 &         0.97 &     \textbf{1.04} &     0.96 &     \textbf{1.01} \\
cars             &         1.04 &     1.04 &         1.05 &     1.05 &     0.36 &     0.36 \\
census           &         0.07 &     \textbf{0.24} &         0.44 &     \textbf{0.47} &     0.30 &     \textbf{0.90} \\
compas           &         1.12 &     \textbf{1.26} &         0.63 &     0.63 &     0.73 &     \textbf{0.96} \\
credit australia &         1.17 &     1.17 &         1.00 &     \textbf{1.01} &     1.04 &     \textbf{1.06} \\
credit german    &         1.01 &     \textbf{1.39} &         0.69 &     \textbf{0.92} &     \textbf{0.72} &     0.60 \\
credit taiwan    &         0.90 &     \textbf{1.04} &         0.74 &     \textbf{1.08} &     1.05 &     \textbf{1.08} \\
heart            &         1.11 &     1.11 &         0.96 &     \textbf{1.06} &     1.22 &     \textbf{1.27} \\
ionosphere       &         0.71 &     0.71 &         0.70 &     0.70 &     1.00 &     \textbf{1.02} \\
kidney           &         0.66 &     0.66 &         0.62 &     0.62 &     0.92 &     \textbf{0.97} \\
krvskp           &         0.53 &     0.53 &         0.54 &     0.54 &     0.65 &     \textbf{0.70} \\
voting           &         0.62 &     0.62 &         0.63 &     0.63 &     1.01 &     \textbf{1.02} \\

\hline\hline
\end{tabular}
\vspace{-2\baselineskip}
\end{minipage}
\end{table}

\subsubsection{Global Explanation Running Time}\label{sec:globalruntime}

The average running time of generating global explanations is reported in Table \ref{tab:globalruntime}.
The running time measures the rule extraction and rule set generation steps for our method, and measures the running time for RuleFit, but excludes the time taken for the model training (e.g., Random Forest) and hyperparameter optimization.
Comparing the methods that share the same base model (RF+ASP and RuleFit, both based on Random Forest), we observe that our method is slower than RuleFit except when the datasets are relatively large (e.g., \textit{adult, census, compas} and \textit{credit taiwan}), and in the latter cases our method can be much faster than RuleFit.
Similar trend is observed for LightGBM, but here in some cases our method was faster than RuleFit (e.g., autism, heart and voting).

\begin{table}[]
    \centering
    \caption{Average running time of generating global explanations, averaged over 5 folds. (Global Explanations)}
    \label{tab:globalruntime}
    \begin{minipage}{\textwidth}
    \begin{tabular}{l r r r r }
    \hline\hline
    Dateset    & \multicolumn{1}{c}{DT\footnote{DT=Decision Tree}+ASP} & \multicolumn{1}{c}{RF\footnote{RF=Random Forest.}+ASP} & \multicolumn{1}{c}{LGBM\footnote{LGBM=LightGBM.}+ASP} & \multicolumn{1}{c}{RuleFit} \\
\hline
    
adult               & 3.07  & 68.24 & 109.08 & 370.17 \\
autism              & 0.01  & 7.09  & 0.00  & 1.80 \\
breast              & 0.06  & 5.11  & 0.30  & 1.82 \\
cars                & 0.18  & 61.59 & 49.52 & 1.79 \\
census              & 13.94 & 569.33 & 60.38 & 2,887.17 \\
compas              & 0.46  & 95.20 & 19.97 & 45.13 \\
credit australia    & 0.01  & 24.87 & 3.14  & 1.76 \\
credit german       & 0.13  & 66.71 & 1.91  & 2.15 \\
credit taiwan       & 0.99  & 123.58 & 481.65 & 250.65 \\
heart               & 0.04  & 73.84 & 1.28  & 3.21 \\
ionosphere          & 0.03  & 12.52 & 0.80  & 1.78 \\
kidney              & 0.02  & 4.93  & 0.83  & 1.74 \\
krvskp              & 0.16  & 74.80 & 35.03 & 2.50 \\
voting              & 0.04  & 12.22 & 0.12  & 1.67 \\

\hline\hline
\end{tabular}
\vspace{-2\baselineskip}
\end{minipage}
\end{table}

\subsection{Evaluating Local Explanations}\label{sec:localexp}
The purpose of generating local explanations is to provide the user with an explanation for the model's prediction for each predicted instance. Here, we use commonly used metrics \textit{local-precision} and \textit{coverage} as proxy measures for the quality of the explanation.\footnote{In the original Anchors paper, the authors use the term \textit{precision}, but here we add \textit{local-} to distinguish from the more commonly used definition of precision.} 
The local-precision compares the (black-box) model predictions of instances covered by the local explanation and the model prediction of the original instance used to induce the local explanation. 
The coverage is the ratio of instances in the validation set that are covered by the local explanation. These two metrics are in a trade-off relationship, where pursuing high coverage is likely to result in low precision explanation and vice versa. 
Furthermore, we also study the number of conditions in the explanation to measure the conciseness of the extracted rules. Additionally, we will also compare the running time to generate the local explanation.

The experiments were carried out similarly to the global explanation evaluation, except that: (1) we replaced RIPPER and RuleFit with Anchors, (2) instead of using the full validation set, we resampled the validation dataset to generate 100 instances in each cross-validation fold for each dataset to estimate the metrics, to complete the experiments in a reasonable amount of time, and (3) in the ASP encoding, we removed the rule selection criteria to avoid excluding rules that are relevant to the predicted instance. 
We were unable to complete Anchors experiment with the \textit{census} dataset due to limited memory (64 GB) on our machine.
For comparison, we computed direct and sufficient explanations with the PyXAI\footnote{\url{https://www.cril.univ-artois.fr/pyxai/}} library, which internally uses SAT and MaxSAT solvers to compute local explanations \cite{audemardExplanatoryPowerBoolean2022,audemardPreferredAbductiveExplanations2022}.
Our method currently does not include the rule simplification feature. 
Therefore, to maintain a consistent comparison, we also deactivated the rule simplification feature in PyXAI. 
Furthermore, as of writing, PyXAI does not yet support LightGBM classifier, so only results for decision tree and random forests are included.
For the running time comparison, we exclude all data preprocessing, training and tree processing, and focus solely on the time taken to generate local explanations.

\begin{landscape}
\begin{table}
    \caption{Decision Tree Local Explanation Metrics}
    \label{tab:localdt}
    \begin{minipage}{\linewidth}
\begin{tabular}{l rrrr rrrr rrrr rrrr}
\hline\hline
{} & \multicolumn{16}{c}{Decision Tree} \\
{} & \multicolumn{4}{c}{\#Conditions} & \multicolumn{4}{c}{Precision} & \multicolumn{4}{c}{Coverage} & \multicolumn{4}{c}{Run Time(s)} \\ \cmidrule(lr){2-5} \cmidrule(lr){6-9} \cmidrule(lr){10-13} \cmidrule(lr){14-17}
Dataset &  Ours &  Anch.\footnote{Anchors.} &  SATd\footnote{PyXAI, direct explanations.} &  SATs\footnote{PyXAI, sufficient explanations.} &  Ours &  Anch. &  SATd &  SATs &  Ours &  Anch. &  SATd &  SATs &  Ours &  Anch. &  SATd*\footnote{Excluded from the run time comparison since the SAT solver is not involved in computing direct explanations.} &  SATs \\
\hline
adult            &         8.65 &           3.31 &           8.65 &               5.32 &              \textbf{1.0} &               0.98 &                \textbf{1.0} &                    \textbf{1.0} &             0.10 &               0.26 &               0.10 &                   0.33 &          0.02 &            1.90 &          0.0003 &              \textbf{0.0014} \\
autism           &         1.00 &           1.55 &           1.00 &               1.00 &              \textbf{1.0} &               \textbf{1.00} &                \textbf{1.0} &                    \textbf{1.0} &             \textbf{0.62} &               0.11 &               \textbf{0.62} &                   \textbf{0.62} &          0.01 &            0.86 &          0.0001 &              \textbf{0.0003} \\
breast           &         4.32 &           1.48 &           4.17 &               3.00 &              \textbf{1.0} &               \textbf{1.00} &                \textbf{1.0} &                    \textbf{1.0} &             0.38 &               0.38 &               0.38 &                   \textbf{0.45} &          0.01 &            1.04 &          0.0002 &              \textbf{0.0005} \\
cars             &         4.03 &           2.20 &           4.03 &               3.07 &              \textbf{1.0} &               0.99 &                \textbf{1.0} &                    \textbf{1.0} &             0.17 &               \textbf{0.21} &               0.17 &                   \textbf{0.21} &          0.01 &            0.22 &          0.0002 &              \textbf{0.0007} \\
census           &         8.08 &            n/a &           8.08 &               3.12 &              \textbf{1.0} &                n/a &                \textbf{1.0} &                    \textbf{1.0} &             0.20 &                n/a &               0.20 &                   \textbf{0.57} &          0.03 &             n/a &          0.0003 &              \textbf{0.0009} \\
compas           &         4.99 &           3.64 &           5.37 &               3.02 &              \textbf{1.0} &               0.99 &                \textbf{1.0} &                    \textbf{1.0} &             0.08 &               0.10 &               0.08 &                   \textbf{0.34} &          0.02 &            0.27 &          0.0002 &              \textbf{0.0010} \\
credit australia &         1.87 &           1.00 &           1.52 &               1.00 &              \textbf{1.0} &               \textbf{1.00} &                \textbf{1.0} &                    \textbf{1.0} &             0.38 &               \textbf{0.51} &               0.38 &                   \textbf{0.51} &          0.01 &            0.25 &          0.0001 &              \textbf{0.0003} \\
credit german    &         4.32 &           3.08 &           4.38 &               3.04 &              \textbf{1.0} &               0.99 &                \textbf{1.0} &                    \textbf{1.0} &             0.14 &               0.27 &               0.14 &                   \textbf{0.31} &          0.01 &            0.92 &          0.0002 &              \textbf{0.0007} \\
credit taiwan    &         5.69 &           2.07 &           6.67 &               2.91 &              \textbf{1.0} &               \textbf{1.00} &                \textbf{1.0} &                    \textbf{1.0} &             0.15 &               0.52 &               0.15 &                   \textbf{0.61} &          0.01 &            1.48 &          0.0002 &              \textbf{0.0008} \\
heart            &         2.60 &           1.85 &           3.24 &               2.37 &              \textbf{1.0} &               \textbf{1.00} &                \textbf{1.0} &                    \textbf{1.0} &             0.23 &               0.31 &               0.23 &                   \textbf{0.40} &          0.01 &            0.33 &          0.0001 &              \textbf{0.0004} \\
ionosphere       &         3.63 &           2.85 &           3.88 &               3.27 &              \textbf{1.0} &               \textbf{1.00} &                \textbf{1.0} &                    \textbf{1.0} &             0.37 &               0.04 &               0.37 &                   \textbf{0.44} &          0.01 &            1.34 &          0.0002 &              \textbf{0.0004} \\
kidney           &         2.84 &           1.69 &           2.49 &               1.83 &              \textbf{1.0} &               \textbf{1.00} &                \textbf{1.0} &                    \textbf{1.0} &             0.30 &               0.10 &               0.30 &                   \textbf{0.41} &          0.01 &            0.88 &          0.0001 &              \textbf{0.0003} \\
krvskp           &         4.78 &           2.62 &           4.78 &               3.66 &              \textbf{1.0} &               0.99 &                \textbf{1.0} &                    \textbf{1.0} &             0.13 &               \textbf{0.17} &               0.13 &                   0.15 &          0.01 &            1.20 &          0.0002 &              \textbf{0.0007} \\
voting           &         2.45 &           1.22 &           2.45 &               1.79 &              \textbf{1.0} &               0.99 &                \textbf{1.0} &                    \textbf{1.0} &             0.33 &               \textbf{0.50} &               0.33 &                   0.48 &          0.01 &            0.37 &          0.0001 &              \textbf{0.0004} \\
\hline\hline
\end{tabular}
\vspace{-2\baselineskip}
\end{minipage}
\end{table}
\end{landscape}

\begin{landscape}
\begin{table}
    \caption{Random Forest Local Explanation Metrics}
    \label{tab:localrf}
    \begin{minipage}{\linewidth}
\begin{tabular}{l rrrr rrrr rrrr rrrr}
\hline\hline
{} & \multicolumn{16}{c}{Random Forest} \\
{} & \multicolumn{4}{c}{\#Conditions} & \multicolumn{4}{c}{Precision} & \multicolumn{4}{c}{Coverage} & \multicolumn{4}{c}{Run Time(s)} \\ \cmidrule(lr){2-5} \cmidrule(lr){6-9} \cmidrule(lr){10-13} \cmidrule(lr){14-17}
Dataset &  Ours &  Anch.\footnote{Anchors.} &  SATd\footnote{PyXAI, direct explanations.} &  SATs\footnote{PyXAI, sufficient explanations.} &  Ours &  Anch. &  SATd &  SATs &  Ours &  Anch. &  SATd &  SATs &  Ours &  Anch. &  SATd*\footnote{Excluded from the run time comparison since the SAT solver is not involved in computing direct explanations.} &  SATs \\
\hline
adult            &         6.37 &           3.07 &         107.51 &              18.80 &             0.77 &               0.99 &               \textbf{1.00} &                   0.99 &             0.14 &               \textbf{0.36} &               0.01 &                   0.03 &          \textbf{0.15} &            5.43 &            0.01 &                0.81 \\
autism           &         3.75 &           2.45 &          62.42 &              11.19 &             0.78 &               \textbf{1.00} &               \textbf{1.00} &                   \textbf{1.00} &             \textbf{0.17} &               0.08 &               0.02 &                   0.11 &          \textbf{0.07} &           12.52 &            0.00 &                0.11 \\
breast           &         3.91 &           2.03 &          65.35 &              23.30 &             0.71 &               \textbf{1.00} &               \textbf{1.00} &                   \textbf{1.00} &             \textbf{0.32} &               0.28 &               0.03 &                   0.09 &          \textbf{0.08} &           25.64 &            0.01 &                0.08 \\
cars             &         6.28 &           2.23 &          20.95 &               4.85 &             0.88 &               0.99 &               \textbf{1.00} &                   0.99 &             0.06 &               \textbf{0.21} &               0.01 &                   0.19 &          \textbf{0.25} &            5.89 &            0.01 &                0.39 \\
census           &         7.86 &            n/a &         241.58 &              10.80 &             0.99 &                n/a &               \textbf{1.00} &                   \textbf{1.00} &             \textbf{0.27} &                n/a &               0.01 &                   0.12 &          \textbf{0.26} &             n/a &            0.02 &               15.70 \\
compas           &         3.74 &           2.88 &         146.55 &              33.15 &             0.78 &               0.98 &               \textbf{1.00} &                   \textbf{1.00} &             \textbf{0.20} &               0.12 &               0.01 &                   0.02 &          \textbf{0.27} &            4.04 &            0.01 &                1.45 \\
credit australia &         3.17 &           2.51 &         139.14 &              38.25 &             0.71 &               0.99 &               \textbf{1.00} &                   \textbf{1.00} &             \textbf{0.25} &               0.22 &               0.02 &                   0.09 &          \textbf{0.10} &           15.89 &            0.00 &                0.12 \\
credit german    &         6.32 &           4.83 &         167.86 &              71.99 &             0.92 &               \textbf{1.00} &               \textbf{1.00} &                   \textbf{1.00} &             0.03 &               \textbf{0.07} &               0.01 &                   0.01 &          \textbf{0.20} &           17.57 &            0.01 &                3.12 \\
credit taiwan    &         7.47 &           1.57 &         479.19 &              42.56 &             0.92 &               0.99 &               \textbf{1.00} &                   \textbf{1.00} &             0.07 &               \textbf{0.61} &               0.01 &                   0.04 &          \textbf{0.20} &            5.00 &            0.02 &                2.30 \\
heart            &         2.22 &           2.89 &         110.21 &              29.06 &             0.71 &               \textbf{0.99} &               \textbf{0.99} &                   \textbf{0.99} &             \textbf{0.21} &               0.14 &               0.04 &                   0.07 &          0.23 &           17.93 &            0.00 &                \textbf{0.22} \\
ionosphere       &         3.95 &           2.94 &         314.19 &             126.37 &             0.81 &               \textbf{1.00} &               \textbf{1.00} &                   \textbf{1.00} &             \textbf{0.35} &               0.03 &               0.02 &                   0.04 &          0.26 &           35.40 &            0.01 &                \textbf{0.20} \\
kidney           &         2.61 &           2.31 &          94.42 &              30.40 &             0.80 &               \textbf{1.00} &               \textbf{1.00} &                   \textbf{1.00} &             \textbf{0.35} &               0.05 &               0.02 &                   0.04 &          \textbf{0.07} &           35.60 &            0.00 &                0.08 \\
krvskp           &         6.11 &           3.78 &          65.17 &              17.53 &             0.80 &               \textbf{1.00} &               \textbf{1.00} &                   \textbf{1.00} &             \textbf{0.14} &               0.13 &               0.01 &                   0.05 &          \textbf{0.23} &           20.87 &            0.01 &                3.94 \\
voting           &         2.27 &           2.23 &          44.46 &               9.86 &             0.91 &               \textbf{1.00} &               \textbf{1.00} &                   \textbf{1.00} &             \textbf{0.42} &               \textbf{0.42} &               0.02 &                   0.19 &          \textbf{0.13} &           27.90 &            0.01 &                0.50 \\
\hline\hline
\end{tabular}
\vspace{-2\baselineskip}
\end{minipage}
\end{table}
\end{landscape}

\begin{table}
    \caption{LightGBM Local Explanation Metrics}
    \label{tab:locallgb}
    \begin{minipage}{\linewidth}
\begin{tabular}{l rr rr rr rr}
\hline\hline
{} & \multicolumn{8}{c}{LightGBM} \\
{} & \multicolumn{2}{c}{\#Conditions} & \multicolumn{2}{c}{Precision} & \multicolumn{2}{c}{Coverage} & \multicolumn{2}{c}{Run Time(s)} \\ \cmidrule(lr){2-3} \cmidrule(lr){4-5} \cmidrule(lr){6-7} \cmidrule(lr){8-9}
Dataset &  Ours &  Anch.\footnote{Anchors.} & Ours &  Anch. & Ours &  Anch. & Ours &  Anch. \\
\hline
adult            &           4.94 &            93.60 &               0.87 &                 \textbf{1.00} &               \textbf{0.79} &                 0.01 &            \textbf{0.35} &             27.92 \\
autism           &           1.00 &             1.55 &               \textbf{1.00} &                 \textbf{1.00} &               \textbf{0.62} &                 0.11 &            \textbf{0.01} &              1.02 \\
breast           &           1.41 &            27.26 &               0.90 &                 \textbf{0.99} &               \textbf{0.60} &                 0.05 &            \textbf{0.07} &              6.66 \\
cars             &           1.17 &            20.46 &               0.72 &                 \textbf{1.00} &               \textbf{0.54} &                 0.01 &            \textbf{0.16} &              1.23 \\
census           &           7.57 &              n/a &               \textbf{0.94} &                  n/a &               \textbf{0.94} &                  n/a &            \textbf{0.34} &               n/a \\
compas           &           3.34 &            21.82 &               0.90 &                 \textbf{1.00} &               \textbf{0.42} &                 0.01 &            \textbf{0.13} &              1.36 \\
credit australia &           2.63 &            35.63 &               0.93 &                 \textbf{1.00} &               \textbf{0.27} &                 0.02 &            \textbf{0.08} &              3.19 \\
credit german    &           2.81 &            51.96 &               0.86 &                 \textbf{1.00} &               \textbf{0.31} &                 0.01 &            \textbf{0.05} &              8.38 \\
credit taiwan    &           3.99 &           141.96 &               0.98 &                 \textbf{1.00} &               \textbf{0.34} &                 0.01 &            \textbf{0.29} &             44.96 \\
heart            &           2.40 &            25.54 &               0.96 &                 \textbf{1.00} &               \textbf{0.21} &                 0.03 &            \textbf{0.02} &              2.22 \\
ionosphere       &           3.50 &            31.10 &               0.98 &                 \textbf{1.00} &               \textbf{0.37} &                 0.02 &            \textbf{0.03} &             14.44 \\
kidney           &           2.19 &            40.74 &               0.89 &                 \textbf{1.00} &               \textbf{0.51} &                 0.02 &            \textbf{0.05} &             11.25 \\
krvskp           &           2.52 &            63.89 &               0.75 &                 \textbf{1.00} &               \textbf{0.51} &                 0.01 &            \textbf{0.28} &             12.25 \\
voting           &           2.06 &            14.78 &               \textbf{1.00} &                 0.98 &               \textbf{0.44} &                 0.05 &            \textbf{0.02} &              3.96 \\
\hline\hline
\end{tabular}
\vspace{-2\baselineskip}
\end{minipage}
\end{table}

\subsubsection{Number of Conditions in Rules}\label{sec:localnumcondition}

Similarly to the evaluation of global explanations (Section \ref{sec:global_num_literals}), we evaluate the number of conditions in local explanations in this section. 
The average number of conditions in rules are listed in Tables \ref{tab:localdt}, \ref{tab:localrf}, \ref{tab:locallgb}.
For the Decision Tree (Table \ref{tab:localdt}) and Random Forest (Table \ref{tab:localrf}), the Anchors produced rules with smaller number of conditions on average overall compared to our method.
As for the LightGBM, the Anchors produced rules with significantly larger number of conditions than our method, and often there was an order of magnitude of difference in the number of rules.
It is possible that the precision guarantee of the Anchors required the algorithm to produce more specific rules, as also indicated by the long run time especially on the datasets that produced the longest rules (e.g., \textit{census, credit taiwan} and \textit{adult}).
This result shows that depending on the underlying learning algorithm, our method can produce shorter and more concise explanations compared to the Anchors.

It is interesting to note that, while PyXAI's direct explanations performed almost identically to ours for Decision Tree, it produced much larger rules for Random Forest.
Similarly, PyXAI's sufficient explanations produced more conditions in the explanations compared to our method or Anchors. 
It is possible that rule simplification could help reduce the number of conditions in such cases.

\subsubsection{Local-Precision and Coverage}\label{sec:localprecision}

The average local-precision, averaged over 5 cross-validation folds, is reported in Tables \ref{tab:localdt}, \ref{tab:localrf}, \ref{tab:locallgb}. 
Note that while Anchors has a minimum precision threshold (we used the default 0.95 setting), ours does not, and indeed we see that all Anchors explanations have higher local-precision than the threshold. 
PyXAI produced the most precise explanations out of the three methods compared, and both direct and sufficient explanations almost always had the perfect local-precision of 1.
The Decision Tree will always have exactly one rule that is relevant to the prediction; therefore, we expect to see exactly 1 local-precision using our method. For the Random Forest and LightGBM, our method produced local explanations with local-precision in 0.8-0.9 range for most of the datasets, but Anchors' explanations had higher local-precision in most cases.

The average coverage, averaged over 5 cross-validation folds, is reported in Tables \ref{tab:localdt}, \ref{tab:localrf}, \ref{tab:locallgb}. 
Interestingly, when using simpler models such as the Decision Tree and Random Forest, Anchors can produce rules that have relatively high coverage, but the pattern does not hold when using a more complex model, which in our case is LightGBM. 
With LightGBM, our method consistently outperformed Anchors in terms of coverage in all datasets, except for the \textit{census} dataset, which we could not run.
For Random Forest, PyXAI's direct explanations are much more precise and apply to a smaller number of instances compared to our method.
Sufficient explanations from PyXAI tended to show greater coverage compared to direct explanations, which aligns with expectations given their fewer number of conditions.

\subsubsection{Local Explanation Running Time}\label{sec:localruntime}

The average running time per instance is reported in Tables \ref{tab:localdt}, \ref{tab:localrf}, \ref{tab:locallgb}. 
For Decision Tree, PyXAI was much faster than both our method and Anchors, whereas for Random Forest, our method was faster than Anchors and PyXAI's sufficient explanations in most datasets. For LightGBM, our method consistently outperformed Anchors in terms of run time.
We also note that our method has a more consistent running time of below 1 second across all datasets, regardless of the complexity of the underlying models, whereas Anchors' running time varies from sub-1 second to tens of seconds, depending on the dataset and model. 
This is likely to be caused by the differences in which these methods query or use information from the original model and generate explanations. 
In fact, a significant amount of time is spent in tree processing in our method, whereas in Anchors the search process is often the most time-consuming step. 
Nonetheless, this comparative experiment demonstrated that our method can produce local explanations in a matter of seconds even when the underlying tree-ensemble is large.

To conclude the experimental section, we summarize the main results obtained in this section. 
For global explanations, we analyzed (1) the average size of generated rule sets and compared it against known methods, as a proxy measure for the degree of simplifications, (2) the relative performance of the rule sets and compared it against known methods, as a proxy measure for the relevance of the explanations, (3) the fidelity of the explanations and (4) the effect of modifying the ASP encoding on the precision metric of the explanations.
Overall, our method was shown to be able to produce smaller rule sets compared to the known methods, however, in terms of the relevance and fidelity of the rules, RuleFit performed better in most cases, demonstrating the trade-off relationship between the complexity of the explanations and performance.

For local explanations, we compared (1) number of conditions, (2) local-precision, (3) coverage and (4) running time of our method against Anchors and PyXAI. 
In terms of local-precision, although our method could produce explanations with reasonably high precision (0.8-0.9 range), Anchors and PyXAI performed better overall.
As for coverage, we found that explanations generated by our method can cover more examples for tree-ensemble. 
Regarding running time, our method had a consistent running time of less than 1 second, whereas the running time of Anchors varied between datasets. 
The experiments for local explanations also highlight the differences between our method and Anchors: while Anchors can produce high-precision rules, our method has an advantage in terms of memory requirement and consistent running time.

\section{Related Works}\label{sec:relatedwork}

Summarizing tree-ensemble models has been studied in literature, see for example, Born Again Trees \cite{breimanBornAgainTrees1996}, defragTrees \cite{haraMakingTreeEnsembles2018} and inTrees \cite{dengInterpretingTreeEnsembles2019}. While exact methods and implementations differ among these examples, a popular approach to tree-ensemble simplification is to create a simplified decision tree model that approximates the behavior of the original tree-ensemble model. Depending on how the approximate tree model is constructed, this could lead to a deeper tree with an increased number of conditions, which makes them difficult to interpret.

Integrating association rule mining and classification is also known, e.g., Class Association Rules (CARs) \cite{liuIntegratingClassificationAssociation1998}, where association rules discovered by pattern mining algorithms are combined to form a classifier. Repeated Incremental Pruning to Produce Error Reduction (RIPPER) \cite{COHEN1995115} was proposed as an efficient approach for classification based on association rule mining, and it is a well-known rule-based classifier. In CARs and RIPPER, rules are mined from data with dedicated association rule mining algorithms, then processed to produce a final classifier.

Interpretable classification models is another area of active research.  Interpretable Decision Sets (IDS) \cite{lakkarajuInterpretableDecisionSets2016} are learned through an objective function, which simultaneously optimizes accuracy and interpretability of the rules. In Scalable Bayesian Rule Lists (SBRL) \cite{yangScalableBayesianRule2017}, probabilistic IF-THEN rule lists are constructed by maximizing the posterior distribution of rule lists. In RuleFit \cite{friedmanPredictiveLearningRule2008}, a sparse linear model is trained with rules extracted from tree-ensembles. RuleFit is the closest to our work in this regard, in the sense that both RuleFit and our method extract conditions and rules from tree-ensembles, but differ in the treatment of rules and representation of final rule sets. In RuleFit, rules are accompanied by regression coefficients, and it is left up to the user to further interpret the result.

\citeN{lundbergLocalExplanationsGlobal2020} showed how a variant of SHAP \cite{lundberg2017unified}, which is a post-hoc explanation method, can be applied to tree-ensembles. 
While our method does not produce importance measures for each feature, the information about which rule fired to reach the prediction can be offered as an explanation in a human-readable format. \citeN{shakerinInductionNonMonotonicLogic2019} proposed a method to use LIME weights \cite{ribeiroWhyShouldTrust2016} as a part of learning heuristics in inductive learning of default theories. Anchors \cite{ribeiroAnchorsHighprecisionModelagnostic2018} generates a single high-precision rule as a local explanation with probabilistic guarantees. 
It should be noted that both LIME and Anchors require the features to be discretized, while recent tree-ensemble learning algorithms can work with continuous features. 
Furthermore, instead of learning rules with heuristics from data, our method directly handles rules which exist in decision tree models with an answer set solver.

There are existing ASP encodings of pattern mining algorithms, e.g., \cite{jarvisaloItemsetMiningChallenge2011a,DBLP:conf/ijcai/GebserGQ0S16,paramonovHybridASPbasedApproach2019}, that can be used to mine itemsets and sequences. Here, we develop and apply our encoding on rules to extract explanatory rules from tree-ensembles. 
On the surface, our problem setting (Section \ref{sec:problemstatement}) may appear similar to frequent itemset and sequence mining; however, rule set generation is different from these pattern mining problems. 
We can indeed borrow some ideas from frequent itemset mining for encoding; however, our goal is not to decompose rules (cf. transactions) into individual conditions (cf. items) then construct rule sets (cf. itemsets) from conditions, but rather to treat each rule in its entirety then combining rules to form rule sets. 
The body (antecedent) of a rule can also be seen as a sequence, where the conditions are connected by conjunction connective \(\wedge\), however, in our case, the ordering of conditions does not matter, thus sequential mining encodings that use slots to represent positional constraints \cite{DBLP:conf/ijcai/GebserGQ0S16} cannot be applied directly to our problem.

Solvers other than ASP solvers have been utilized for similar tasks.
For example, \citeN{yuComputingOptimalDecision2020} proposed SAT- and MaxSAT-based approaches to minimize the total number of conditions used in the target decision set.
Their approaches construct interpretable decision sets based on SAT- and MaxSAT-encodings, instead of using a weighted objective function  \cite{lakkarajuInterpretableDecisionSets2016} that contains multiple terms such as coverage, number of rules and conditions.
\citeN{chenRobustnessVerificationTreebased2019} proposed an efficient algorithm for the robustness verification of tree-ensemble models, which surpasses existing MILP (mixed integer linear programming) methods in terms of speed. 
While they do not consider (local) explanations explicitly in their setting, their method allows computation of anchor features such that changes outside these features cannot change the prediction.
More recently, alternative methods to generate explanations based on logical definitions have been proposed \cite{izzaExplainingRandomForests2021,ignatievUsingMaxSATEfficient2022,audemardExplanatoryPowerBoolean2022,audemardPreferredAbductiveExplanations2022}.
These methods focus on local explanations conforming to logical conditions, such as abductive (sufficient) explanations. 
Similar to our approach, these methods are tailored for tree-ensemble models, and utilize SAT and MaxSAT solvers for efficient processing.
On the other hand, our method differs from these methods by using heuristics for generating explanations, and, by leveraging the flexibility of ASP, facilitates the inclusion of user-defined selection criteria and preferences.
Regarding model-agnostic explanation methods, while Anchors is model-agnostic, its reliance on sampling to construct explanations often results in a longer run time, as exemplified by experimental results reported in Section \ref{sec:localexp}.

\citeN{gunsItemsetMiningConstraint2011a} applied constraint programming (CP), a declarative approach, to itemset mining. This constraint satisfaction perspective led to the development of ASP encoding of pattern mining  \cite{jarvisaloItemsetMiningChallenge2011a,guyetUsingAnswerSet}. \citeN{DBLP:conf/ijcai/GebserGQ0S16} applied preference handling to sequential pattern mining, and \citeN{paramonovHybridASPbasedApproach2019} extended the declarative pattern mining by incorporating dominance programming (DP) from \citeN{negrevergneDominanceProgrammingItemset2013} to the specification of constraints. \citeN{paramonovHybridASPbasedApproach2019} proposed a hybrid approach where the solutions are effectively screened first with dedicated algorithms for pattern mining tasks, then declarative ASP encoding is used to extract condensed patterns. While the aforementioned works focused on extracting interesting patterns from transaction or sequence data, our focus in this paper is to generate rule sets from tree-ensemble models to help users interpret the behavior of machine learning models. As for the ASP encoding, we use dominance relations similar to the ones presented in \citeN{paramonovHybridASPbasedApproach2019} to further constrain the search space.

\section{Conclusion}\label{sec:conclusion}
In this work, we presented a method for generating rule sets as global and local explanations from tree-ensemble models using pattern mining techniques encoded in ASP. Unlike other explanation methods that focus exclusively on either global or local explanations, our two-step approach allows us to handle both global and local explanation tasks. 
We showed that our method can be applied to two well-known tree-ensemble learning algorithms, namely Random Forest and LightGBM. Evaluation on various datasets demonstrated that our method can produce explanations with good quality in a reasonable amount of time, compared to existing methods.

Adopting the declarative programming paradigm with ASP allows the user to take advantage of the expressiveness of ASP in representing constraints and optimization criteria. 
This makes our approach particularly suitable for situations where fast prototyping is required, since changing the constraint and optimization settings require relatively low effort compared to specialized pattern mining algorithms. 
Useful explanations can be generated using our approach, and combined with the expressive ASP encoding, we hope that our method will help the users of tree-ensemble models to better understand the behavior of such models.

A limitation of our method in terms of scalability is the size of search space, which is exponential in the number of valid rules.
%In practical applications where model performance is paramount, restricting the parameters of the induced trees (such as depth or number of leaves) in the ensemble algorithm is undesirable.
When the number of candidate rules is large, we suggest using stricter individual rule constraints on the rules, or reducing the maximum number of rules to be included into rule sets (Section \ref{sec:encodingconstraints}), to achieve reasonable solving time.
Another limitation is the lack of rule simplification in the generation of explanations, since more straightforward rules could enhance the user's comprehension.
Furthermore, while the current ASP encoding considers the overlap between rule sets with the same consequent class (Section \ref{sec:optimizingrulesets}), it does not consider the overlaps between the two rule sets with different consequent classes.

% One of the possible future directions is to incorporate rule simplification methods, which could lead to rules with improved condition literals. In future, we plan to explore how ASP and modern statistical machine learning could be integrated effectively to produce more interpretable machine learning systems.

% First, improvement of the ASP encodings is necessary to be able to scale to even larger tree-ensembles. Second, incorporating rule simplification approaches could lead to rules with improved condition literals. Third, it would be interesting to apply KR techniques such as default reasoning and belief update and study how they could be integrated into the current framework. In future, we plan to explore how ASP and modern statistical machine learning could be integrated effectively to produce more interpretable machine learning systems.

There are a number of directions for further research. 
First, while the current work did not modify the conditions in the rules in any way, rule simplification approaches could be incorporated to remove redundant conditions. 
Second, we could extend the current work to support regression problems. 
Third, further research might explore alternative approaches to implement a model-agnostic explanation method, for example, by combining a sampling-based local search strategy with a rule selection component implemented with ASP.
In addition, while the multi-objective optimization approach (Section \ref{sec:optimizingrulesets}) allows for incorporating user desiderata, the fidelity to the original models can still be improved. 
Future works could focus on exploring alternative encodings or additional optimization strategies to better capture the nuances of the original models' decision-making processes, thereby improving the effectiveness of the explanations.
Furthermore, although local and global explanations serve different purposes and may not always align perfectly (Section \ref{sec:rulesetgenerationforglobal}), achieving a certain level of consistency is important for maintaining the credibility of the explanations.
Future research could explore methods to reconcile such differences, thereby creating a more unified model explanation framework.
More generally, in the future, we plan to explore how ASP and modern statistical machine learning could be integrated effectively to produce more interpretable machine learning systems.

\section*{Acknowledgements}
This work has been supported by JSPS KAKENHI Grant Number JP21H04905 and JST CREST Grant Number JPMJCR22D3, Japan.

\clearpage

\appendix

\section{Additional tables}\label{sec:appendix_table}

The accuracy, F1 scores, precision and recall of the base models after hyperparameter optimization are shown in Tables \ref{tab:basemodelaccuracy} and \ref{tab:basemodelprecision}. 
The values in this table are used as the denominators when calculating the performance ratio in Table \ref{tab:performanceratio}.

\begin{table}[ht]
    \centering
    \caption{Base model accuracy and F1 scores, averaged over 5 folds.}
    \label{tab:basemodelaccuracy}
    \begin{minipage}{\textwidth}
    \begin{tabular}{l rrrrr rrrrr }
    \hline\hline
{} & \multicolumn{5}{c}{Accuracy} & \multicolumn{5}{c}{F1 score} \\ \cmidrule(lr){2-6} \cmidrule(lr){7-11}
model            &     DT\footnote{DecisionTree} &  RF\footnote{RandomForest} & LGBM\footnote{LightGBM} & R.Fit\footnote{RuleFit} & Rip.\footnote{RIPPER} &  DT &  RF & LGBM & R.Fit & Rip. \\ 
\midrule
adult            &     0.86 &  0.85 &   0.87 &      0.86 &     0.84 &  0.66 &  0.61 &   0.71 &      0.69 &     0.63 \\
autism           &     1.00 &  1.00 &   1.00 &      1.00 &     1.00 &  1.00 &  0.99 &   1.00 &      1.00 &     1.00 \\
breast           &     0.93 &  0.97 &   0.97 &      0.95 &     0.94 &  0.90 &  0.96 &   0.95 &      0.93 &     0.91 \\
cars             &     0.97 &  0.98 &   1.00 &      0.99 &     0.93 &  0.95 &  0.97 &   1.00 &      0.99 &     0.89 \\
census           &     0.95 &  0.94 &   0.96 &      0.96 &     0.95 &  0.43 &  0.05 &   0.60 &      0.56 &     0.50 \\
compas           &     0.78 &  0.80 &   0.80 &      0.80 &     0.79 &  0.57 &  0.55 &   0.58 &      0.59 &     0.59 \\
credit australia &     0.86 &  0.87 &   0.87 &      0.87 &     0.85 &  0.85 &  0.85 &   0.85 &      0.85 &     0.84 \\
credit german    &     0.70 &  0.76 &   0.76 &      0.72 &     0.73 &  0.42 &  0.44 &   0.50 &      0.52 &     0.50 \\
credit taiwan    &     0.81 &  0.82 &   0.82 &      0.82 &     0.82 &  0.46 &  0.43 &   0.47 &      0.47 &     0.46 \\
heart            &     0.76 &  0.83 &   0.78 &      0.82 &     0.78 &  0.71 &  0.79 &   0.74 &      0.79 &     0.74 \\
ionosphere       &     0.88 &  0.93 &   0.94 &      0.92 &     0.88 &  0.91 &  0.95 &   0.95 &      0.94 &     0.91 \\
kidney           &     0.95 &  1.00 &   0.99 &      1.00 &     0.97 &  0.96 &  1.00 &   1.00 &      1.00 &     0.97 \\
krvskp           &     0.99 &  0.97 &   0.99 &      0.99 &     0.99 &  0.99 &  0.97 &   0.99 &      0.99 &     0.99 \\
voting           &     0.97 &  0.96 &   0.96 &      0.95 &     0.95 &  0.97 &  0.97 &   0.97 &      0.96 &     0.96 \\

\hline\hline
\end{tabular}
\vspace{-2\baselineskip}
\end{minipage}
\end{table}

\begin{table}[htb]
    \centering
    \caption{Base model precision and recall, averaged over 5 folds.}
    \label{tab:basemodelprecision}
    \begin{minipage}{\textwidth}
    \begin{tabular}{l rrrrr rrrrr }
    \hline\hline
{} & \multicolumn{5}{c}{Precision} & \multicolumn{5}{c}{Recall} \\ \cmidrule(lr){2-6} \cmidrule(lr){7-11}
model            &     DT &  RF & LGBM & R.Fit & Rip. &  DT &  RF & LGBM & R.Fit & Rip. \\ 
\midrule
adult            &      0.77 &  0.82 &   0.78 &      0.77 &     0.73 &   0.58 &  0.49 &   0.66 &      0.62 &     0.55 \\
autism           &      1.00 &  1.00 &   1.00 &      1.00 &     1.00 &   1.00 &  0.99 &   1.00 &      1.00 &     1.00 \\
breast           &      0.92 &  0.94 &   0.94 &      0.93 &     0.94 &   0.88 &  0.97 &   0.96 &      0.93 &     0.89 \\
cars             &      0.95 &  0.95 &   1.00 &      0.99 &     0.87 &   0.95 &  0.99 &   1.00 &      0.99 &     0.92 \\
census           &      0.69 &  0.97 &   0.75 &      0.74 &     0.69 &   0.32 &  0.03 &   0.50 &      0.46 &     0.39 \\
compas           &      0.62 &  0.71 &   0.69 &      0.68 &     0.64 &   0.53 &  0.45 &   0.50 &      0.53 &     0.55 \\
credit australia &      0.79 &  0.85 &   0.86 &      0.85 &     0.81 &   0.93 &  0.86 &   0.84 &      0.86 &     0.87 \\
credit german    &      0.50 &  0.71 &   0.64 &      0.55 &     0.56 &   0.37 &  0.32 &   0.42 &      0.50 &     0.45 \\
credit taiwan    &      0.65 &  0.69 &   0.67 &      0.67 &     0.66 &   0.35 &  0.31 &   0.36 &      0.36 &     0.36 \\
heart            &      0.74 &  0.83 &   0.79 &      0.83 &     0.78 &   0.69 &  0.77 &   0.70 &      0.76 &     0.71 \\
ionosphere       &      0.90 &  0.92 &   0.93 &      0.91 &     0.90 &   0.91 &  0.98 &   0.98 &      0.97 &     0.92 \\
kidney           &      0.95 &  1.00 &   1.00 &      1.00 &     0.97 &   0.96 &  1.00 &   0.99 &      1.00 &     0.98 \\
krvskp           &      0.98 &  0.96 &   0.99 &      0.99 &     0.99 &   1.00 &  0.98 &   1.00 &      0.99 &     0.99 \\
voting           &      0.99 &  0.98 &   0.98 &      0.96 &     0.97 &   0.96 &  0.95 &   0.95 &      0.96 &     0.95 \\

\hline\hline
\end{tabular}
\vspace{-2\baselineskip}
\end{minipage}
\end{table}

\section{Hyperparameter optimization}\label{sec:appendix_hyperparameter}
All hyperparameters were optimized using \textit{optuna} \cite{akibaOptunaNextgenerationHyperparameter2019}. 
As evaluation metric, we chose the F1 score. 
Hyperparameter tuning was performed separately for each fold on the training data. 
Within each fold, we used 20\% of the training data as the validation set. 
We used early stopping for hyperparameter search, where the search was terminated after 30 trials if the validation metric did not improve for 30 consecutive rounds. 
We set the maximum number of search trials to 200, and the time-out for each study was set to 1,200 seconds. 
The search range for each model is shown in Table \ref{tab:hyperparameter}.  

\begin{table}[!h]
    \centering
    \caption{Search space definition for hyperparameter optimization}
    \label{tab:hyperparameter}
    \begin{tabular}{llll}
    \toprule
    Model/Parameter & Type & Value Range & Step \\
    \midrule
    Decision Tree & \\
    \midrule
    max\_depth    & integer & [2, 9] & \\
    min\_samples\_leaf  & float & [0.01, 0.2] & \\
    min\_weight\_fraction\_leaf & float & [0.0, 0.5] & 0.01 \\
    criterion & categorical & [gini, entropy] & \\
    \midrule
    Random Forest & \\
    \midrule
    n\_estimators & integer & [50, 500] & 10 \\
    max\_depth    & integer & [2, 9] \\
    min\_samples\_leaf  & float & [0.01, 0.2] & \\
    min\_weight\_fraction\_leaf & float & [0.0, 0.5] & 0.01 \\
    criterion & categorical & [gini, entropy] & \\
    \midrule
    LightGBM \\
    \midrule
    objective (fixed)   & categorical   & binary            &\\
    metric (fixed)      & categorical   & binary logloss    &\\
    num\_boost\_round (fixed) & integer       & 1000              &\\
    early\_stopping (fixed)   & integer       & 30    &\\
    learning\_rate      & float         & [0.01, 0.2] &  \\
    max\_depth          & integer       & [2, 9] \\
    num\_leaves         & integer       & [2, 100] \\
    min\_data\_in\_leaf & integer       & [10, 500] & 10 \\
    min\_child\_weight  & float         & [0.001, 10] &  \\
    feature\_fraction   & float         & [0.05, 1.0] &  \\
    subsample           & float         & [0.2, 1.0]    &  \\
    subsample\_freq     & int           & [1, 20] & \\
    lambda\_l1          & float         & [1e-5, 10]    &  \\
    lambda\_l2          & float         & [1e-5, 10]    &  \\
    \midrule
    RuleFit \\
    \midrule
    rule\_generator & categorical & random forest \\
    memory\_parameter & float & [0.0, 1.0] & 0.1 \\
    lin\_standardise & boolean \\
    lin\_trim\_quantile & boolean \\
    \midrule
    RIPPER \\
    \midrule
    num\_folds & integer & [2, 5] & 1 \\
    prune & boolean \\
    no\_error\_check & boolean &\\
    \bottomrule
    \end{tabular}
\end{table}

\clearpage

%%%%%%%%%%%%%%

\bibliographystyle{acmtrans}
\bibliography{tplp_treetap}

\label{lastpage}
\end{document}